\documentclass{article}
\usepackage[final]{style}

\usepackage[utf8]{inputenc} 
\usepackage[T1]{fontenc}    
\usepackage{url}            
\usepackage{booktabs}       
\usepackage{amsfonts}       
\usepackage{nicefrac}       
\usepackage{microtype}      
\usepackage{wrapfig,lipsum,booktabs}
\usepackage{multicol}
\usepackage{enumitem}

\usepackage{graphicx}
\usepackage{subfigure} 
\usepackage{algorithm}
\usepackage{algorithmic}
\usepackage{bm,color}
\usepackage{mathdots}
\usepackage{float}
\usepackage{amsmath, amssymb,amsthm}
\usepackage{xspace}
\usepackage{array}
\usepackage{tabu}
\usepackage{subfigure}
\graphicspath{ {./}{./plots/} }
\usepackage{capt-of}

\usepackage{caption}
\usepackage{capt-of}
\usepackage{ mathrsfs }

\newcommand{\cocoa}{\textsc{CoCoA}\xspace} 
\newcommand{\cocoap}{\textsc{CoCoA$\!^{\bf \textbf{\footnotesize+}}$}\xspace}
\newcommand{\proxcocoa}{\textsc{prox}\cocoap}

\newcommand{\glmnet}{\textsc{glmnet}\xspace}
\newcommand\tagthis{\addtocounter{equation}{1}\tag{\theequation}}

\newcommand{\Lone}{L_1}
\newcommand{\Ltwo}{L_2}

\DeclareMathOperator*{\argmin}{arg\,min}

\newcommand{\eqdef}{:=}
\newcommand{\R}{\mathbb{R}}                      
\newcommand{\E}{\mathbb{E}}                      
\newcommand{\Exp}{\mathbb{E}}                      

\newcommand{\diag}{\mathbf{diag}}

\newcommand{\gap}{G}

\newcommand{\xv}{ {\bf x}}
\newcommand{\yv}{ {\bf y}}

\newcommand{\uv}{ {\bf u}}
\newcommand{\vv}{ {\bf v}}
\newcommand{\wv}{ {\bf w}}
\newcommand{\alphav}{ {\boldsymbol \alpha}}

\newcommand{\bv}{ {\bf b}}

\newcommand{\0}{ {\bf 0}}

\newcommand{\aggpar}{\gamma}

\newcommand{\vsubset}[2]{#1_{[#2]}}

\newcommand{\vc}[2]{#1^{(#2)}}                   

\newcommand{\norm}[1]{\left\lVert{#1}\right\rVert}

\newcommand{\bP}{\mathcal{P}}
\newcommand{\bD}{\mathcal{D}}

\newcommand{\Ggk}{\mathcal{G}^{\sigma'}_k\hspace{-0.08em}}

\newtheorem*{rep@theorem}{\rep@title}
\newcommand{\newreptheorem}[2]{%
\newenvironment{rep#1}[1]{%
 \def\rep@title{#2 \ref{##1}}%
 \begin{rep@theorem}}%
 {\end{rep@theorem}}}
\newreptheorem{lemma}{Lemma'}
\newreptheorem{definition}{Definition'}
\newreptheorem{proposition}{Proposition'}
\newreptheorem{theorem}{Theorem'}

\theoremstyle{plain}
\newtheorem{theorem}{Theorem}
\newtheorem{lemma}[theorem]{Lemma}

\newtheorem{assumption}{Assumption}

\newtheorem{remark}{Remark}

\theoremstyle{definition}
\newtheorem{definition}{Definition}

\usepackage{hyperref}       

\let\oldell\ell 
\renewcommand{\ell}{g}

\title{$\Lone$-Regularized Distributed Optimization: \\ A Communication-Efficient Primal-Dual Framework}

\author{
  Virginia Smith \\
  UC Berkeley 
   \And
   Simone Forte \\
   ETH Z\"{u}rich$^{\dagger}$ \\ 
  \AND 
   Michael I. Jordan \\
  UC Berkeley \\
   \And
  Martin Jaggi \\
  ETH Z\"{u}rich \\
}

\begin{document}

\maketitle

\vspace{-1em}
\begin{abstract}
Despite the importance of sparsity in many large-scale applications, there are few methods for distributed optimization of sparsity-inducing objectives. In this paper, we present a communication-efficient framework for  $L_1$-regularized optimization in the distributed environment. By viewing classical objectives in a more general primal-dual setting, we develop a new class of methods that can be efficiently distributed and applied to common sparsity-inducing models, such as Lasso, sparse logistic regression, and elastic net-regularized problems. We provide theoretical convergence guarantees for our framework, and demonstrate its efficiency and flexibility with a thorough experimental comparison on Amazon EC2. Our proposed framework yields speedups of up to 50$\times$ as compared to current state-of-the-art methods for distributed $\Lone$-regularized optimization.
\end{abstract}

\section{Introduction}
\label{intro}
\nocite{Forte:2015wv}

In this paper, we consider standard regularized loss minimization problems, including as our main focus $\Lone$-regularized optimization problems of the form
\[
  \min_{\alphav \in\R^n} \ f(A \alphav) + \lambda \norm {\alphav}_1 \, ,
\]
where $\alphav \in \R^n$ is the weight vector, $A \in \R^{d \times n}$ is a given data matrix, and $\lambda$ is a regularization parameter. This formulation includes many popular $L_1$-regularized classification and regression models, such as Lasso and sparse logistic regression, and is easily extended to other separable regularizers like elastic net. Models of this form are particularly useful in high-dimensional settings because of their tendency to bias learning towards sparse solutions. However, despite their importance, few methods currently exist to efficiently fit such sparsity-inducing models in the distributed environment.

One promising distributed method is \cocoap \cite{Jaggi:2014vi,Ma:2015ti}, a recently proposed primal-dual framework that demonstrates competitive performance, provides a flexible communication scheme, and enables the use of off-the-shelf single-machine solvers internally. 
However, by solving the problem in the dual, \cocoap (like SDCA, prox-SDCA, and numerous other primal-dual methods \cite{ShalevShwartz:2013wl, ShalevShwartz:2014dy,Yang:2013vl,Zhang:2015vj,Zheng16}) is only equipped to handle strongly convex regularizers, which prevents it from being directly applied to $\Lone$-regularized objectives. Moreover, by requiring the data to be distributed by data point rather than by feature, communication can become a prohibitive bottleneck for \cocoap as the number of features grows large, which is precisely the setting of interest for $L_1$ regularization.

In this work, we take a different perspective and propose a framework that can run either in the dual, or on the primal directly. From this change in perspective we derive several new primal-dual distributed optimization methods, in particular for sparsity-inducing regularizers. Our approach uses ideas from \cocoap, though leveraging these ideas in this new setting requires significant theoretical and algorithmic modifications, particularly in handling non-strongly convex regularizers. The proposed primal-dual framework and associated rates are novel contributions even in the non-distributed case.

\subsection{Contributions}\vspace{-1mm}

\paragraph{Generalized framework.}
By building on the \cocoap framework, \proxcocoa comes with several benefits, including the use of arbitrary local solvers on each machine, and the analysis of and ability to solve subproblems to arbitrary accuracies. However in contrast to \cocoap, we consider a much broader class of optimization problems. This results in a more general framework that: (1) specifically incorporates the case of $\Lone$ regularization; (2) allows for the flexibility of distributing the data by either feature or data point; and (3) can be run on either the primal or dual formulation, which we show to have significant theoretical and practical implications.
\vspace{-2mm}

\paragraph{Analysis of non-strongly convex regularizers and losses.}
We derive convergence rates for the general class of problems considered in this work, leveraging a novel approach in the analysis of primal-dual rates for non-strongly convex regularizers. The proposed technique is a significant improvement over simple smoothing techniques used in, e.g., \cite{Nesterov:2005ic,ShalevShwartz:2014dy,Zhang:2015vj} that enforce strong convexity by adding a small $L_2$ term to the objective. Our results include primal-dual rates and certificates for both strongly convex and non-strongly convex regularizers and losses, and we show how earlier rates of \cocoa and \cocoap can be derived as a special case of our new rates / methods.
\vspace{-5mm}

\paragraph{Experimental comparison.}
The proposed framework yields order-of-magnitude speedups (as much as 50$\times$ faster) as compared to other state-of-the-art methods for $L_1$-regularized optimization. We demonstrate these performance gains in an extensive experimental comparison on real-world distributed datasets. We additionally show significant improvements over \cocoap when considering strongly convex objectives. All algorithms for comparison are implemented in \textsf{\small Apache Spark} and run on Amazon EC2 clusters. Our code is available at: \href{http://github.com/gingsmith/proxcocoa}{\texttt{github.com/gingsmith/proxcocoa}}.
\vspace{-1mm}

\vspace{-1mm}
\section{Setup}
\label{sec:setup}

\vspace{-1mm}
A great variety of methods in machine learning and signal processing are posed as the minimization of a weighted sum of two convex functions, where the first term is a convex function of a linear predictor and the second term is a regularizer:
\begin{equation}
\label{eq:obj}\tag{A}
  \min_{\alphav \in\R^n} \ f(A \alphav) 
    + g(\alphav) \, .
\end{equation}
Here $\alphav \in \R^n$ is the parameter vector, and $A := [ \xv_1; \dots; \xv_n ] \in \R^{d \times n}$ is a data matrix with column vectors $\xv_i\in\R^d$, $i\in [n]$ and row vectors $\yv_j^T \in \R^n$, $j\in [d]$.
Our central assumption will be that $g(\cdot)$ is \emph{separable}, meaning that\vspace{-3mm}
\[
g(\alphav) = \sum_{i=1}^{n} \ell_i(\alpha_i)\vspace{-1mm} 
\]
for convex functions $\ell_i : \R\rightarrow\R$. 
Furthermore, we assume $f : \R^d \rightarrow \R$ is $(1/{\tau})$-smooth for~$\tau>0$. \vspace{-2mm}

\paragraph{Examples.} 
The above setting encompasses all convex loss functions depending on linear predictors $\yv_j^T \alphav$, together with most common convex regularizers, including all separable functions, such as $\Lone$- or general $L_p$-norms, or the elastic net 
given by $\frac{\eta}{2} \norm{\cdot}_2 +  (1-\eta) \norm{\cdot}_1$. 
\vspace{-2mm}

\paragraph{Data partitioning.}
To map this setup to the distributed environment, we suppose that the dataset~$A$ is distributed over $K$ machines according to a partition $\{\mathcal{P}_k\}_{k=1}^K$ of the \emph{columns} of $A \in \R^{d\times n}$.
We denote the size of the partition on machine $k$ by $n_k=|\mathcal{P}_k|$. For $k\in[K]$ and $\alphav\in \R^n$, we define $\vsubset{\alphav}{k}\in \R^n$ as the $n$-vector with elements $(\vsubset{\alphav}{k})_i := \alpha_i$ if $i\in \mathcal{P}_k$ and $(\vsubset{\alphav}{k})_i := 0$ otherwise. 

\section{The \proxcocoa Algorithmic Framework}
\label{sec:framework}
The \proxcocoa framework is given in Algorithm~\ref{alg:generalizedcocoa}. This framework builds on the recent \cocoap framework \cite{Jaggi:2014vi,Ma:2015ti}, though with a more general objective, a modified subproblem, and where we allow the method to be applied to either the primal or dual formulation. 
To distribute the method, we assign each machine to work only on local coordinates of the weight vector $\alphav$, and access only data that is stored locally. Machines share state through the vector $\vv := A\alphav$. This vector is communicated at each round after using local solvers in parallel to find (possibly) approximate solutions to the subproblems defined in~\eqref{eq:subproblem}. 
Solving the primal problem~\eqref{eq:obj} directly with \proxcocoa will result in distributing the data column-wise (by feature), and having the vector $\vv$ be of length equal to the number of data points. This can greatly reduce communication costs as the number of features grows (see Section~\ref{sec:experiments}). Most importantly, the proposed setup will prepare us to handle non-strongly convex regularizers in both theory and practice, as we further explain in the following sections. 

\paragraph{Data-local quadratic subproblems.}
For each machine, we define a data-local subproblem of the original optimization problem \eqref{eq:obj}. This simpler problem can be solved on machine $k$ and only requires accessing data which is already available locally, i.e., columns $A_i$ 
such that $i\in\mathcal{P}_k$. The subproblem depends only on the previous shared vector $\vv := A\alphav$ and the local data: 
\begin{equation} 
\min_{\vsubset{\Delta \alphav}{k}\in\R^{n}} \ 
\Ggk(  \vsubset{\Delta \alphav}{k}; \vv, \vsubset{\alphav}{k}) \, , \vspace{-1mm}
\end{equation} 
where
\begin{align*}
\hspace{-.95em}
\Ggk(  \vsubset{\Delta \alphav}{k}; \vv, \vsubset{\alphav}{k})
\eqdef  \frac{1}{K} f(\vv) 
 + \wv^T A\vsubset{\Delta \alphav}{k}
+\frac{\sigma'}{2\tau}  \Big\|A\vsubset{\Delta \alphav}{k}\Big\|^2 + \sum_{i \in \mathcal{P}_k} 
\ell_i(\alphav_i + {\vsubset{\Delta \alphav}{k}}_i)
\tagthis
 \label{eq:subproblem}
\end{align*}
with $\wv := \nabla f( \vv )$. We denote the change of local variables~$\alpha_i$ for indices $i\in\mathcal{P}_k$ as $\vsubset{\Delta \alphav}{k}$.
For a given aggregation parameter $\gamma \in (0,1]$, the subproblem relaxation parameter $\sigma'$ will be set as $\sigma' := \gamma K$, but can also be improved in a data-dependent way as we discuss in Appendix~\ref{sec:convergenceproofs}.

\setlength{\textfloatsep}{13pt}
\begin{algorithm}[t]
\caption{\proxcocoa Distributed Framework for Problem \eqref{eq:obj}}
\label{alg:generalizedcocoa}
\begin{algorithmic}[1]
\STATE {\bf Input:} Data matrix $A$ distributed column-wise according to partition $\{\mathcal{P}_k\}_{k=1}^K$, aggregation parameter $\aggpar\!\in\!(0,1]$, 
and parameter $\sigma'$ for the local subproblems
$\Ggk(  \vsubset{\Delta \alphav}{k}; \vv, \vsubset{\alphav}{k})$.\\
Starting point $\vc{\alphav}{0} := \0 \in \R^n$, $\vc{\vv}{0}:=\0\in \R^d$.
\FOR {$t = 0, 1, 2, \dots $}
  \FOR {$k \in \{1,2,\dots,K\}$ {\bf in parallel over computers}}
     \STATE call local solver, returning a $\Theta$-approximate solution 
     $\vsubset{\Delta \alphav}{k}$   
        of  the local subproblem~\eqref{eq:subproblem}
     \STATE update local variables $\vsubset{\vc{\alphav}{t+1}}{k} := \vsubset{\vc{\alphav}{t}}{k} + \aggpar \, \vsubset{\Delta \alphav}{k}$
     \STATE return updates to shared state $\Delta \vv_k :=  A 
     						\vsubset{\Delta \alphav}{k}$
  \ENDFOR
  \STATE reduce 
$\vc{\vv}{t+1}  := \vc{\vv}{t} +
  \aggpar \textstyle \sum_{k=1}^K \Delta \vv_k $
\ENDFOR 
\end{algorithmic}
\end{algorithm}

\vspace{-.35em}
\paragraph{Reusability of existing single-machine solvers.}
Our local subproblems have the appealing property of being very similar in structure to the global problem~\eqref{eq:obj}, with the main difference being that they are defined on a smaller (local) subset of the data.
For the user of our framework, this presents a major advantage in that existing single machine-solvers can be directly re-used in our distributed framework (Algorithm~\ref{alg:generalizedcocoa}) by employing them on the subproblems~$\Ggk$.
Therefore, problem-specific tuned solvers which have already been developed, along with associated speed improvements (such as multi-core implementations), can be easily leveraged in the distributed setting. We quantify the dependence on local solver performance in more detail in our convergence analysis (Section~\ref{sec:convergence}). 
\vspace{-.5em}
\paragraph{Interpretation.} The above definition of the local objective functions $\Ggk$ are such that they closely approximate the global objective in \eqref{eq:obj} as the ``local'' variable~$\vsubset{\Delta \alphav}{k}$ varies, which we will see in the analysis (Lemma~\ref{lem:RelationOfDTOSubproblems} in the appendix).
In fact, if the subproblem were solved exactly, this could be interpreted as a data-dependent, block-separable proximal step, applied to the $f$ part of the objective~\eqref{eq:obj} as follows:
\begin{align*}
\sum_{k=1}^K 
 \Ggk(\vsubset{\Delta \alphav}{k}; \vv, \vsubset{\alphav}{k}) 
\ =\  L + f(\vv) + \nabla f( \vv )^TA\Delta \alphav 
+ \frac{\sigma'}{2\tau} \Delta \alphav^T 
\begin{bmatrix}
    A_{[1]}^TA_{[1]}\vspace{-1mm} &   & 0 \\
      &\hspace{-2mm}\ddots&  \\
    0 &   &\hspace{-2mm}A_{[K]}^TA_{[K]}
\end{bmatrix}
\Delta \alphav \, ,
\end{align*}
where $L = \sum_{i \in [n]} \ell_i(\alphav_i + \Delta \alphav_i) \, .$

However, note that in contrast to traditional proximal methods, our algorithm does \emph{not} assume that the prox subproblems be solved to high accuracy, as we instead allow the use of local solvers of any approximation quality $\Theta$. This notion is made precise with the following assumption. \vspace{.25em}

\begin{assumption}[$\Theta$-approximate solution, see \cite{Ma:2015ti}]
\label{asm:theta}
We assume that  there exists $\Theta \in [0,1)$ such that 
$\forall k\in [K]$, 
the local solver at any outer iteration $t$ produces
a (possibly) randomized approximate solution $\vsubset{\Delta \alphav}{k}$,
which satisfies
\begin{align}
\label{eq:localSolutionQuality}
 \E \big[
&\Ggk(\vsubset{\Delta \alphav}{k};\vv, \vsubset{\alphav}{k})\!
-
 \Ggk(\vsubset{\Delta \alphav^{\star}}{k};\vv, \vsubset{\alphav}{k})
\big]  
\!\leq \Theta
\left(
 \Ggk({\bf 0};\vv, \vsubset{\alphav}{k})
 -
 \Ggk(\vsubset{\Delta \alphav^{\star}}{k};\vv, \vsubset{\alphav}{k})
 \right) \, ,
\end{align}

where
\vspace{-2.5em}

\begin{align}
\label{eq:asjfcowjfcaw}
\vsubset{\Delta \alphav^{\star}}{k}
\in \argmin_{\Delta \alphav \in \R^n} \ 
 \Ggk(\vsubset{\Delta \alphav}{k};\vv, \vsubset{\alphav}{k}) \hspace{2mm} \forall k\in[K] \, . 
\end{align}
\end{assumption} 

\begin{remark}\label{rem:localtime}
In practice, the time spent solving the local subproblems in parallel should be chosen comparable to the required time of a communication round, for best overall efficiency on a given system. 
We study this trade-off both in theory (Section \ref{sec:convergence}) and experiments (Section \ref{sec:experiments}).
\end{remark}

\vspace{-.5em}
\subsection{Primal-Dual Context}
\label{sec:primaldual}
Exploiting primal-dual structure is not a requirement to optimize \eqref{eq:obj}; indeed, we have shown above how to solve this optimization problem directly. However, noting the relationship between primal and dual objectives has many benefits, including computation of the duality gap, which allows us to have a certificate of approximation quality. 
It is also useful as an analysis tool and helps relate this work to the prior work of \cite{Yang:2013vl,Jaggi:2014vi,Ma:2015ti}. To leverage this structure, starting from our original formulation \eqref{eq:obj} with objective function $\bD(\alphav) := f(A\alphav )  + \sum_{i=1}^{n} \ell_i(\alpha_i)$, the \emph{dual problem} is given by
\vspace{-1mm}
\begin{equation}
    \label{eq:dualP}\tag{B}
    \min_{\wv \in \R^{d}} \quad \Big[ \
    \bP(\wv) := f^*(\wv ) +
    \sum_{i=1}^{n} \ell^*_i(-\xv_i^T\wv) \ \Big] \, . 
\end{equation}

\vspace{-3mm}
Here $\wv \in \R^d$ is a weight vector and $\xv_i \in \R^d$ are columns of the data matrix $A$. The functions $f^*,\ell^*_i$ are the \textit{convex conjugates} of $f,\ell_i$ in the original problem \eqref{eq:obj}. This duality structure is known as Fenchel-Rockafellar Duality (see \citep[Theorem 4.4.2]{Borwein:2005ge} or a self-contained derivation in the appendix).

Given $\alphav \in \R^{n}$ in the context of \eqref{eq:obj}, a corresponding primal vector $\wv\in \R^d$ for problem \eqref{eq:dualP} is obtained by: \vspace{-2mm}
\begin{equation}
\label{eq:dualPdualrelation}
\wv = \wv(\alphav) := \nabla f( A\alphav ) \, .
\end{equation}
This mapping is given by the first-order optimality conditions for the $f$-part of the objective. (Recall that we assumed $\ell_i : \R \rightarrow \R$ are arbitrary closed convex functions, $f : \R^d \rightarrow \R$ is $(1/{\tau})$-smooth.)
The duality gap, given by:\vspace{-1mm}
\begin{equation}
\label{eq:gap}
G(\alphav) := \bP(\wv(\alphav)) - (-\bD(\alphav))
\end{equation}
acts as a certificate of approximation quality, as the distance to the true optimum $\bP(\wv^{\star})$ is always bounded above by the duality gap. A globally defined and finite duality gap $G(\alphav)$ for any problem~\eqref{eq:obj} can be obtained by bounding the region of interest for the iterates $\alphav$. This ``Lipschitzing'' trick will make the conjugates $\ell^*_i$ globally defined and Lipschitz \cite{Dunner:2016vga}, as we prove in Section~\ref{sec:convergence}.

\vspace{-.5em}
\paragraph{Primal vs. Dual.} 
Previous work of \cocoap mapped machine learning tasks to $\bP(\wv)$~\eqref{eq:dualP}, and then solved this problem in the dual. While this can still be accomplished with the machinery of \proxcocoa (see Section~\ref{sec:oldcocoa}), here our main focus is to instead solve the original objective $\bD(\alphav)$~\eqref{eq:obj} directly. 
This can have a large practical impact for the described applications in the distributed setting, as it implies that we can distribute the data by \textit{feature} rather than by data point. Further, we will communicate a vector equal in size to the number of data points, as opposed to the number of features. When the number of features is high (as is common in sparsity-inducing models) this can significantly reduce communication and improve overall performance, as we demonstrate in Section~\ref{sec:experiments}. Further, it allows us to directly leverage state-of-the-art coordinate-wise primal methods, such as \glmnet~\cite{Friedman:2010wm} and extensions \cite{Yuan:2012wi,Johnson:2015tq}. From a theoretical perspective, solving $\bD(\alphav)$ will allow us to consider non-strongly convex regularizers, which were not covered in \cocoap, as we discuss in Section~\ref{sec:convergence}.

\section{Convergence Analysis} \label{sec:convergence}

In this section we provide convergence rates for the proposed framework, and introduce an important theoretical technique in analyzing non-strongly convex terms in the primal-dual setting. For simplicity of presentation, we assume in the analysis that the data partition is balanced; i.e., $n_k = n/K$ for all $k$. Furthermore, we assume that the columns of A satisfy $\| \xv_i \| \le 1$ for all $i \in [n]$.  We present results for the case where $\gamma:=1$ in Algorithm~\ref{alg:generalizedcocoa}, and where the subproblems~\eqref{eq:subproblem} are defined using the corresponding safe bound $\sigma':=K$. This case delivers the fastest convergence rates in the distributed setting, which in particular don't degrade as the number of machines $K$ grows and $n$ remains fixed. 

\vspace{-.5em}
\subsection{General Convex $\ell_i$}
Our first main theorem provides convergence guarantees for objectives with non-strongly convex regularizers, including models such as Lasso and sparse logistic regression. Providing primal-dual rates and globally defined primal-dual accuracy certificates requires a theoretical technique that we introduce below, in which we show how to satisfy the following notion of $L$-bounded support. 

\begin{definition}[$L$-Bounded Support]
\label{def:lbounded}
A function $h$ has \emph{$L$-bounded support} if its effective domain is bounded by $L$, i.e.,\vspace{-2mm}
\begin{equation}
h(\uv) < + \infty  \ \Rightarrow \  \|\uv\| \le L \, .
\end{equation}
\end{definition}
As we explain in Section \ref{sec:oldcocoa} of the appendix, our assumption of $L$-bounded support for the $\ell_i$ functions can be interpreted as an assumption that their conjugates are globally $L$-Lipschitz.\vspace{.25em}

\begin{theorem}
\label{thm:convergenceNonsmooth}
Consider Algorithm \ref{alg:generalizedcocoa} with $\gamma :=1$, 
and let $\Theta$ be the quality of the local solver as in Assumption~\ref{asm:theta}.
Let~$\ell_i$ have $L$-bounded support, 
and $f$ be $(1/{\tau})$-smooth. 
Then after $T$ iterations where\vspace{-2mm}
\begin{align}\label{eq:dualityRequirements}
T
\geq
T_0 +
& \max\{\Big\lceil \frac1{1-\Theta}\Big\rceil,
\frac{4L^2n^2}{\tau \epsilon_{G}(1-\Theta)}\} \, ,
\\
T_0
\geq t_0+
\Big[
\frac{2}{ 1-\Theta }
\left(\frac {8L^2n^2} {\tau \epsilon_{ G}}
-1
\right)
\Big]_+ & \, , \, \, \,
t_0  \geq
  \max(0,\Big\lceil \tfrac1{(1-\Theta)}
\log \left(
\tfrac{
 \tau({\bD}(\vc{\alphav}{0})-{\bD}(\alphav^{\star} ))
  }{2 L^2 Kn}
  \right)
 \Big\rceil)\, ,\notag
\end{align}
we have that the expected duality gap satisfies 
\[
\Exp[\bP( \wv(\overline\alphav)) - (-\bD(\overline \alphav)) ] \leq \epsilon_{ G} \, ,
\]
where $\overline\alphav$ is the averaged iterate returned by Algorithm~\ref{alg:generalizedcocoa}.
\end{theorem}

\paragraph{Bounded support modification.} 

Note that the absolute value function $\ell_i=|\cdot|$ for $\Lone$ regularization does not have $L$-bounded support, and thus violates the assumptions yielding convergence in Theorem~\ref{thm:convergenceNonsmooth}. Its dual, the indicator function of the interval, is not defined globally, and thus does not always allow a finite duality gap.
To address this, existing approaches typically use a simple smoothing technique as in \cite{Nesterov:2005ic}: by adding a small amount of $\Ltwo$ to the $\Lone$-norm, it becomes strongly convex; see, e.g.,~\cite{ShalevShwartz:2014dy}. This Nesterov smoothing technique is undesirable in practice, as it changes the iterates, the convergence rate, and the tightness of the resulting duality gap. Further, the amount of smoothing can be difficult to tune and can have a large influence on the performance of the method at hand. We show examples of this issue with experiments in Section~\ref{sec:experiments}. 

In contrast, our approach preserves all solutions of the original objective, leaves the iterate sequence unchanged, and allows for direct reusability of existing $L_1$ solvers. It also removes the need for additional parameter tuning.
To achieve this, we modify the function $|\cdot|$ by imposing an additional weak constraint 
 that is inactive in our region of interest. Formally, we replace $\ell_i(\cdot) = |\cdot|$ by\vspace{-.5mm}
\[
\bar{\ell_i}(\alpha) :=
\begin{cases}
|\alpha| &: \alpha \in [-B,B] \\
+\infty &: \text{otherwise.}
\end{cases}
\]
For large enough $B$, this problem yields the same solution as the original $\Lone$-regularized objective. Note that this only affects convergence theory, in that it allows us to present a strong primal-dual rate (Theorem~\ref{thm:convergenceNonsmooth} for $L$=$B$). The modification of~$\ell_i$ does not affect the algorithms for the original problems. Whenever a monotone optimizer is used, we will never leave the level set defined by the objective at the starting point. We provide further details on this technique in Section~\ref{sec:appendix-conjugates}, and illustrate how to leverage it for a variety of applications (see Section~\ref{sec:applications} of the appendix and also \cite{Dunner:2016vga}).

\subsection{Strongly Convex $\ell_i$}

For the case of strongly convex $\ell_i$, including elastic net-regularized objectives, we obtain the following faster \emph{geometric} convergence rate. \vspace{1mm}
\setlength{\belowdisplayskip}{2pt} \setlength{\belowdisplayshortskip}{2pt}
\setlength{\abovedisplayskip}{2pt} \setlength{\abovedisplayshortskip}{2pt}
\begin{theorem}
\label{thm:convergenceSmooth}
Consider Algorithm~\ref{alg:generalizedcocoa} with $\gamma := 1$,
and let $\Theta$ be the quality of the local solver as in Assumption~\ref{asm:theta}.
Let~$\ell_i$ be  $\mu$-strongly convex $\forall i\in[n]$, 
and~$f$ be $(1/{\tau})$-smooth.
Then we have that~$T$ iterations are sufficient for suboptimality $\epsilon_\bD$, with \vspace{-1.2em}
 
\begin{equation}
T \geq
\tfrac{1}
   {
\gamma(1-\Theta)}
\tfrac
{\mu\tau+n}
{ \mu \tau}
    \log \tfrac n {\epsilon_\bD} \, . 
\end{equation}
Furthermore, after $T$ iterations with
\[
 T 
    \geq 
\tfrac{1}
   {\aggpar
(1-\Theta)}
\tfrac
{\mu\tau+
n}
{ \mu \tau}
    \log 
\left(
\tfrac{1}
   {\aggpar
(1-\Theta)}
\tfrac
{\mu\tau+
n}
{ \mu \tau}
    \tfrac n {\epsilon_\gap}
    \right) \, ,
\]
\[
\Exp[
\bP( \wv(\vc{\alphav}{T})) - (-\bD(\vc{\alphav}{T}))
]\leq \epsilon_\gap \, .
\]
\end{theorem}

We provide proofs of both Theorem~\ref{thm:convergenceNonsmooth} and Theorem~\ref{thm:convergenceSmooth} in the appendix (Section~\ref{sec:convergenceproofs}).

\section{Related Work}\vspace{-1mm}
\label{sec:relatedwork}

\paragraph{Single-machine coordinate solvers.}
For strongly convex regularizers, current state-of-the-art for empirical loss minimization is randomized coordinate ascent on the dual (SDCA)~\cite{ShalevShwartz:2013wl} and its accelerated variants, e.g., \cite{ShalevShwartz:2014dy}. In contrast to primal stochastic gradient descent (SGD) methods, the SDCA family is often preferred as it is free of learning-rate parameters and has faster (geometric) convergence guarantees.  
Interestingly, a similar trend in coordinate solvers has been observed in recent Lasso literature, but with the roles of primal and dual reversed. For those problems, coordinate descent methods on the primal have become state-of-the-art, as in \glmnet~\cite{Friedman:2010wm} and extensions \cite{Yuan:2012wi}; see, e.g., the overview in \cite{Yuan:2010ub}. However, primal-dual convergence rates for unmodified coordinate algorithms have to our knowledge been obtained only for strongly convex regularizers to date \cite{ShalevShwartz:2014dy,Zhang:2015vj}.

\vspace{-.75em}
\paragraph{Connection to coordinate-wise Newton methods.}
Coordinate descent on $\Lone$-regularized problems~\eqref{eq:obj} with $g(\cdot)=\lambda\|\cdot\|_1$ can be interpreted as the iterative minimization of a quadratic approximation of the smooth part of the objective (as in a one-dimensional Newton step), followed by a shrinkage step resulting from the $\Lone$ part. In the single-coordinate update case, this is at the core of \glmnet~\cite{Friedman:2010wm,Yuan:2010ub}, and widely used in, e.g., solvers based on the primal formulation of $\Lone$-regularized objectives \cite{ShalevShwartz:2011vo,Yuan:2012wi,Bian:2013wx,Fercoq:2015kd,Tappenden:2015vha}. When changing more than one coordinate at a time, again employing a quadratic upper bound on the smooth part, this results in a two-loop method as in \glmnet \cite{Friedman:2010wm} for the special case of logistic regression.
This idea is crucial for the distributed setting.

\vspace{-.75em}
\paragraph{Parallel coordinate descent.}
Parallel coordinate descent for $\Lone$-regularized objectives (with and without using mini-batches) was proposed in~\cite{Bradley:2011wq} (Shotgun) and generalized in~\cite{Bian:2013wx}
, and is among the best performing solvers in the parallel setting.
Our framework reduces to Shotgun as a special case when the internal solver is a single coordinate update on the subproblem \eqref{eq:subproblem}, $\gamma=1$, and for a suitable~$\sigma'$. However, Shotgun is not covered by our convergence theory, since it uses a potentially un-safe 
upper bound~$\beta$ instead of $\sigma'$, which isn't guaranteed to satisfy the condition \eqref{eq:sigmaPrimeSafeDefinition}. 
Other parallel coordinate descent methods on the $\Lone$-objective have recently been analyzed in \cite{Fercoq:2015kd,Tappenden:2015vha,Necoara:2016cr}, but not in the communication-efficient or distributed setting.

\vspace{-.75em}
\paragraph{Distributed solvers.}

The methods most closely related to our approach are distributed variants of \glmnet as in \cite{Mahajan:2014tg}. Inspired by \glmnet and \cite{Yuan:2012wi}, the work of \cite{Bian:2013wx,Mahajan:2014tg} introduced the idea of a block-diagonal Hessian upper approximation in the distributed $\Lone$ context.
The later work of \cite{Trofimov:2014vb} specialized this approach to sparse logistic regression.

If hypothetically each of our quadratic subproblems $\Ggk(\vsubset{\Delta \alphav}{k})$ as defined in \eqref{eq:subproblem} were to be minimized exactly, the resulting steps could be interpreted as block-wise Newton-type steps on each coordinate block~$k$, where the Newton-subproblem is modified to also contain the $\Lone$-regularizer \cite{Mahajan:2014tg,Yuan:2012wi,Qu:2015ve}. While \cite{Mahajan:2014tg} allows a fixed accuracy for these subproblems---but not arbitrary approximation quality $\Theta$ as in our framework---the work of
\cite{Trofimov:2014vb,Yuan:2012wi,Yen:2015vy} assumes that the quadratic subproblems are solved exactly. Therefore, these methods are not able to freely trade off communication and computation. Also, they do not allow the re-use of arbitrary local solvers. On the theoretical side, the rate results provided by \cite{Mahajan:2014tg,Trofimov:2014vb,Yuan:2012wi} are not explicit convergence rates but only asymptotic, as the quadratic upper bounds are not explicitly controlled for safety as with our $\sigma'$.

\vspace{-.5em}
\paragraph{Batch solvers.}
ADMM \cite{Boyd:2010bw}, proximal gradient descent, and quasi-Newton methods such as L-BFGS and are also often used in distributed environments because of their relatively low communication requirements. However, they require at least a full (distributed) batch gradient computation at each round, and therefore do not allow the gradual trade-off between communication and computation provided by \proxcocoa. The works of \cite{McMahan:2013ux}
and \cite{Kang:2014aa} have obtained encouraging results for distributed systems employing coordinate descent variants on $\Lone$-problems. The latter approach distributes both columns and rows of the data matrix and can be extended to Lasso. However it only provides asymptotic improvement per step, and no convergence rate. We include experimental comparisons with ADMM, prox-GD, and orthant-wise limited memory quasi-Newton (OWL-QN) \cite{Andrew:2007cu}, an L-BFGS variant that can handle $\Lone$ regularization~\cite{Yu:2010vw}, but which has no convergence rate.

Finally, we note that while the provided convergence rates for \proxcocoa mirror the convergence class of classical batch gradient methods in terms of the number of outer rounds, existing batch proximal gradient methods come with a weaker theory, as they do not allow general inexactness $\Theta$ for the local subproblem~\eqref{eq:subproblem}. 
In contrast, our shown convergence rates incorporate this approximation directly, and, moreover, hold for arbitrary local solvers of much cheaper cost than batch methods (where in each round, every machine has to process exactly a full pass through the local data). This makes \proxcocoa more flexible in the distributed setting, as it can adapt to varied communication costs on real systems. We will see in the following section that this flexibility results in significant performance gains over the competing methods.

\vspace{-1mm}
\section{Experimental Results}\label{sec:experiments}

In this section we compare \proxcocoa to numerous state-of-the-art methods for large-scale $L_1$-regularized optimization, including: \vspace{2mm}

\begin{itemize}[leftmargin=0.3in]
\setlength\itemsep{0mm}
\begin{minipage}{0.55\linewidth}
\setlength{\listparindent}{0.1in}
\item \textsc{Mb-SGD}: mini-batch stochastic gradient 

\hspace{3mm} descent with an $L_1$-prox
\item \textsc{Prox-GD}: full proximal gradient descent
\item \textsc{OWL-QN}: orthant-wise limited quasi-Newton
\end{minipage}
\begin{minipage}{0.45\linewidth} 
\item \textsc{ADMM}: alternating direction method 

\hspace{3mm} of multipliers
\item \textsc{Mb-CD}: mini-batch parallel coordinate  

\hspace{3mm} descent, incl. \textsc{Shotgun}
\end{minipage}
\end{itemize} 
The first three methods are optimized and implemented in Apache Spark's MLlib (v1.5.0) \cite{Meng:2015tu}. 
We employ coordinate descent as a local solver for \proxcocoa, and apply \proxcocoa directly to the primal formulation of Lasso and elastic net, thereby mapping the problem to~\eqref{eq:obj} and solving this objective directly.
A comparison with \textsc{Shotgun} is provided as an extreme case to highlight the detrimental effects of frequent communication in the distributed environment. 

We test the performance of each method in large-scale experiments fitting Lasso and elastic net regression models to the datasets shown in Table~\ref{tab:datasets}. All code is written in \textsf{\small Apache Spark} and experiments are run on public cloud Amazon EC2 m3.xlarge machines with one core per machine. For \textsc{Mb-CD}, \textsc{Shotgun}, and \proxcocoa in the primal, datasets are distributed by feature, whereas for \textsc{Mb-SGD}, \textsc{Prox-GD}, \textsc{OWL-QN}, \textsc{ADMM}, and \cocoap they are distributed by datapoint.

\begin{wraptable}{r}{.46\textwidth}\vspace{-.5em}
\captionof{table}{Datasets for Empirical Study\vspace{-1mm}}
\label{tab:datasets}
   \begin{center}
      \begin{tabular}{l| r |  
      r | r }
    {\small\textbf{Dataset}} & {\small\textbf{Training}} &
    {\small\textbf{Features}} & {\small\textbf{Sparsity}}   \\
    \hline
	url & 2 M & 3 M & 3.5e-5 \\ 
	epsilon & 400 K & 2 K & 1.0 \\
	kddb & 19 M & 29 M & 9.8e-7 \\
	webspam & 350 K & 16 M & 2.0e-4 \\
      \end{tabular}\vspace{-1em}
   \end{center}
   \end{wraptable}

We carefully tune each competing method for best performance. \textsc{ADMM} requires the most tuning, both in selecting the penalty parameter $\rho$ and in solving the subproblems. Solving the subproblems to completion for ADMM is prohibitively slow, and we thus use iterations of conjugate gradient and improve performance by allowing early stopping. We also use a varying penalty parameter~$\rho$ --- practices described in \citep[Sec. 4.3, 3.4.1]{Boyd:2010bw}. For \textsc{Mb-SGD}, we tune the step size and mini-batch size parameters. For \textsc{Mb-CD}, we scale the updates at each round by $\frac{\beta}{b}$ for mini-batch size $b$ and $\beta \in [1,b]$, and tune both parameters $b$ and $\beta$. Further implementation details for all methods are given in the appendix (Section~\ref{sec:expdetails}).

\newcommand{\smalltrimfig}[1]{\subfigure{\includegraphics[trim = 30 180 60 180, clip, width=.246\linewidth]{#1}}}
\begin{figure*}[t!]
\smalltrimfig{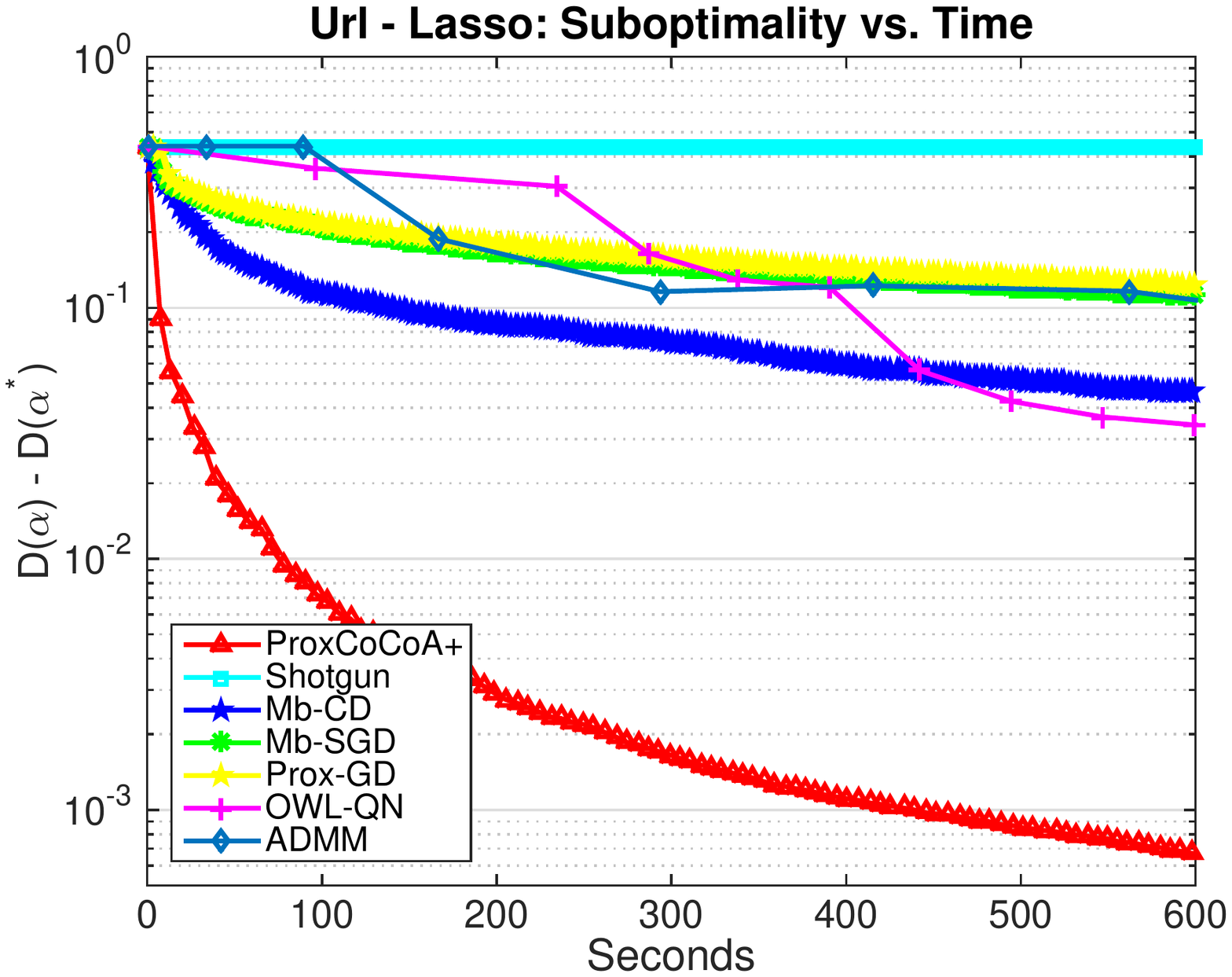}
\smalltrimfig{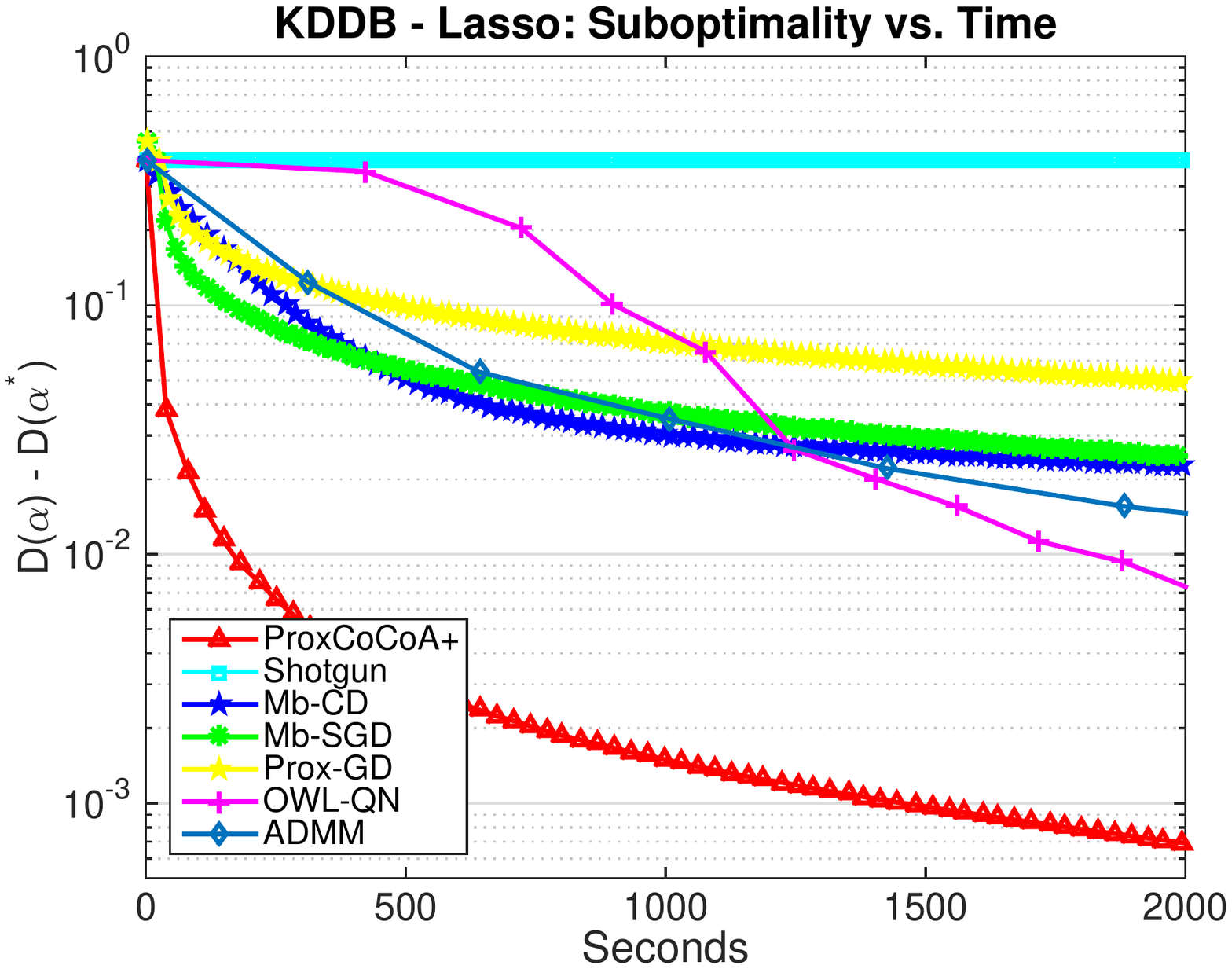}
\smalltrimfig{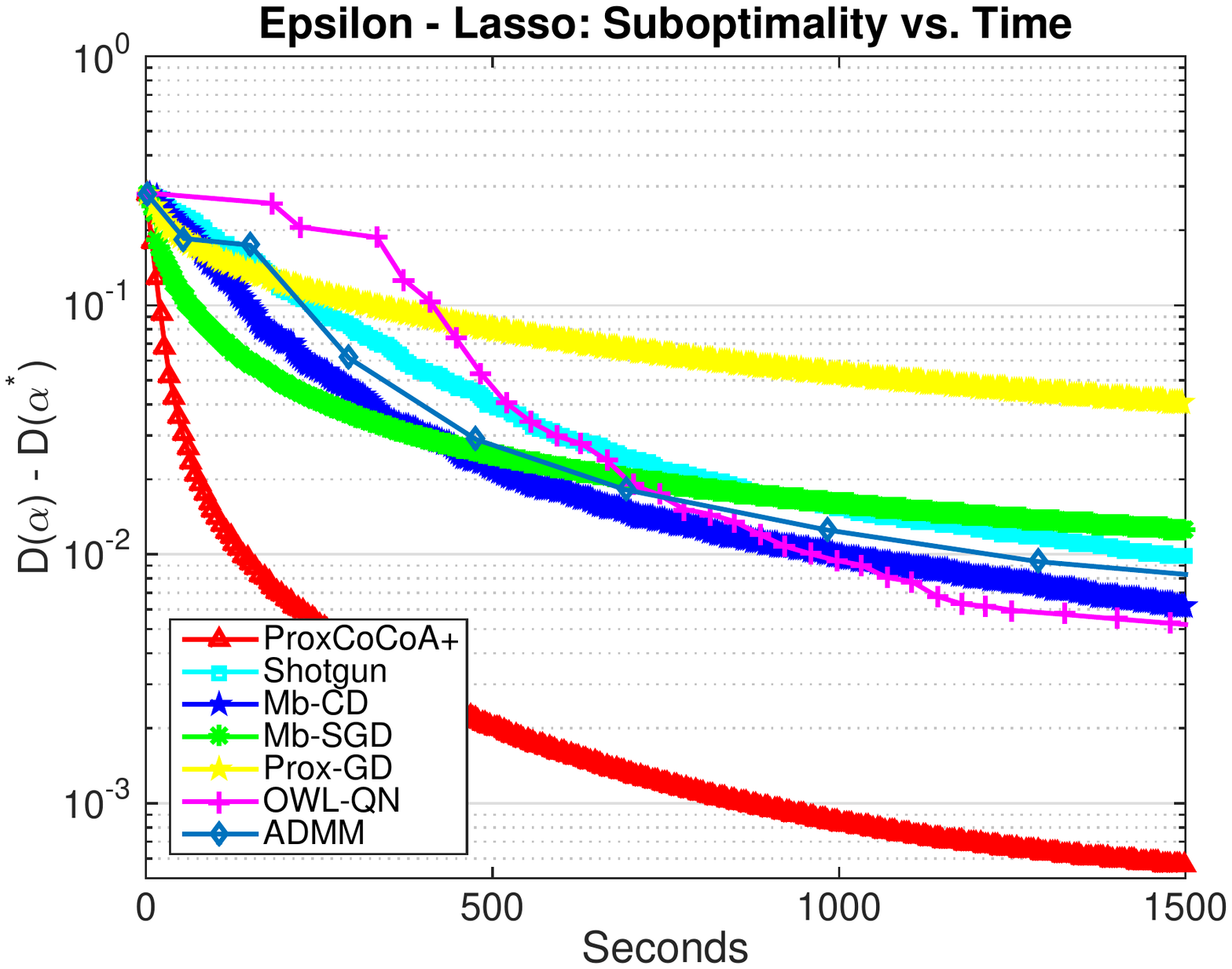}
\smalltrimfig{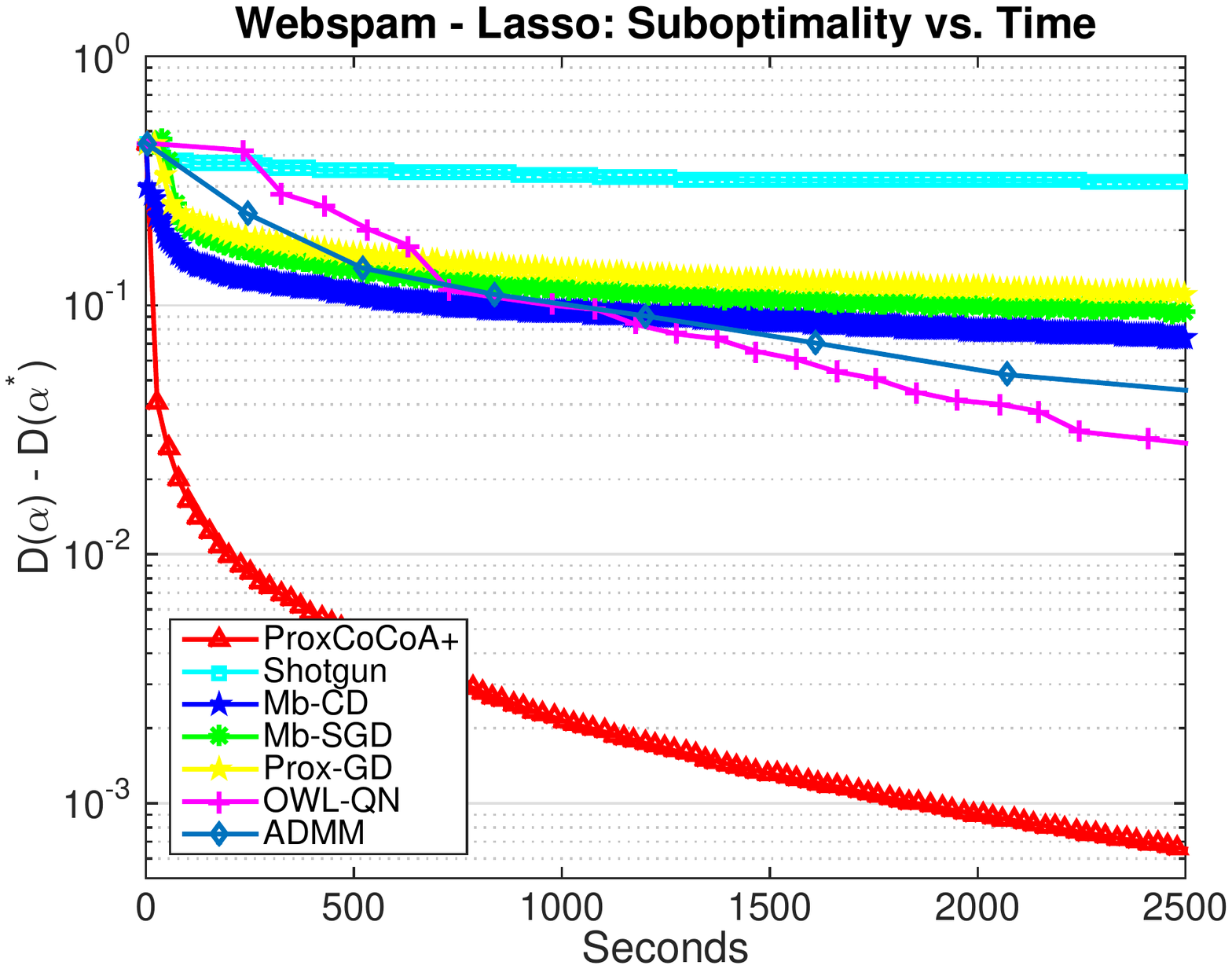}\vspace{-1em}
\caption{\small{Suboptimality in terms of ${\bD}(\alphav)$ for solving Lasso regression for: url ($K$=4, $\lambda$=\mbox{1\sc{e}-4}), kddb ($K$=4, $\lambda$=\mbox{1\sc{e}-6}), epsilon ($K$=8, $\lambda$=\mbox{1\sc{e}-5}), and webspam ($K$=16, $\lambda$=\mbox{1\sc{e}-5}) datasets. \proxcocoa applied to the primal formulation converges more quickly than all other compared methods in terms of the time in seconds.}}
\label{fig:comparison}
\end{figure*}

\vspace{-2mm}
\paragraph{Comparison with $L_1$ methods.} In analyzing the performance of each algorithm (Figure \ref{fig:comparison}), we measure the improvement to the primal objective given in \eqref{eq:obj} $({\bD}(\alphav))$ in terms of 
wall-clock time in seconds. We see that both \textsc{Mb-SGD} and \textsc{Mb-CD} are slow to converge, and come with the additional burden of having to tune extra parameters (though \textsc{Mb-CD} makes clear improvements over \textsc{Mb-SGD}). 
As expected, naively distributing \textsc{Shotgun}~\cite{Bradley:2011wq} (single coordinate updates per machine) does not perform well, as it is tailored to shared-memory systems and requires communicating too frequently. 
\textsc{OWL-QN} performs the best of all compared methods, but is still much slower to converge than \proxcocoa, converging, e.g., 50$\times$ more slowly for the webspam dataset. The optimal performance of \proxcocoa is particularly evident in datasets with large numbers of features (e.g., url, kddb, webspam), which are exactly the datasets of interest for $\Lone$ regularization. 

\newcommand{\tinytrimfig}[1]{\subfigure{\includegraphics[trim = 45 180 65 180, clip, width=.49\linewidth]{#1}}}
\begin{minipage}[h]{.48\textwidth}
\captionsetup{type=figure}
\tinytrimfig{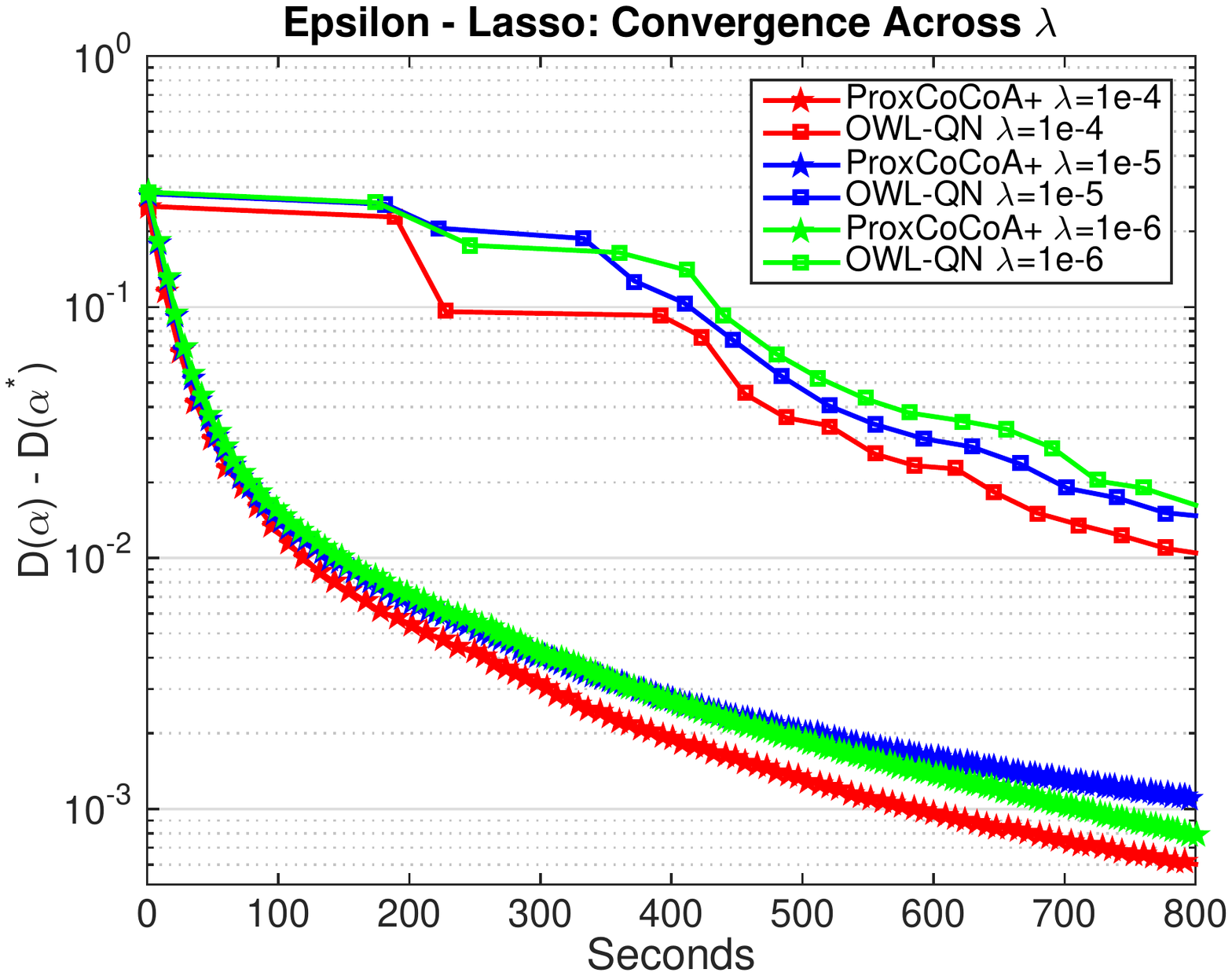}
\tinytrimfig{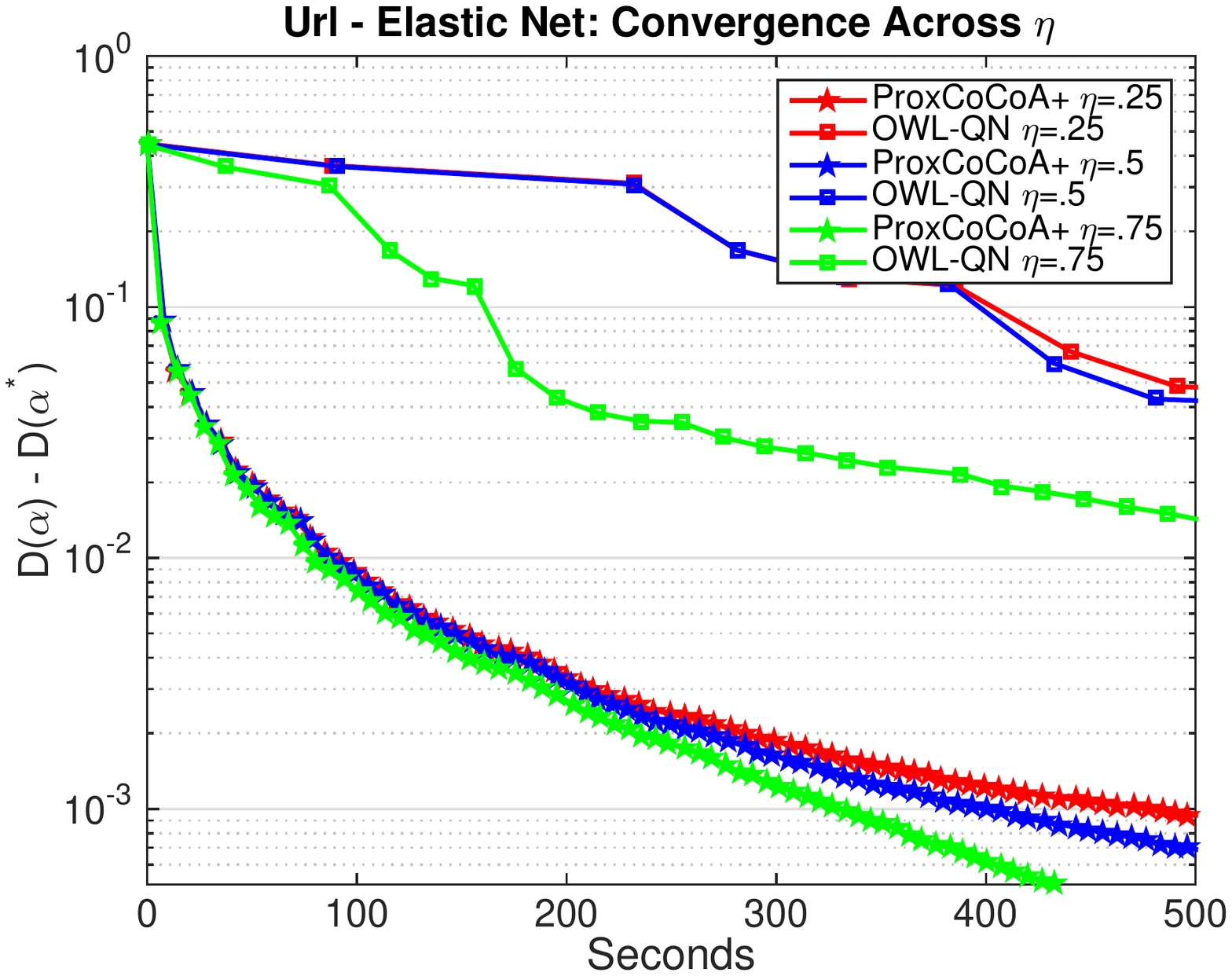}
\vspace{-1.5em}
\caption{\small{Suboptimality in terms of ${\bD}(\alphav)$ for solving Lasso for the epsilon dataset (left, $K$=8) and elastic net for the url dataset, (right, $K$=4, $\lambda$=\mbox{1\sc{e}-4}). Speedup are robust over different regularizers $\lambda$ (left), and across problem settings, including varying $\eta$ parameters of elastic net regularization (right).}}
\label{fig:lambda}
\end{minipage} \hspace{1mm}
\begin{minipage}[h]{.48\textwidth}
\captionsetup{type=figure}
\tinytrimfig{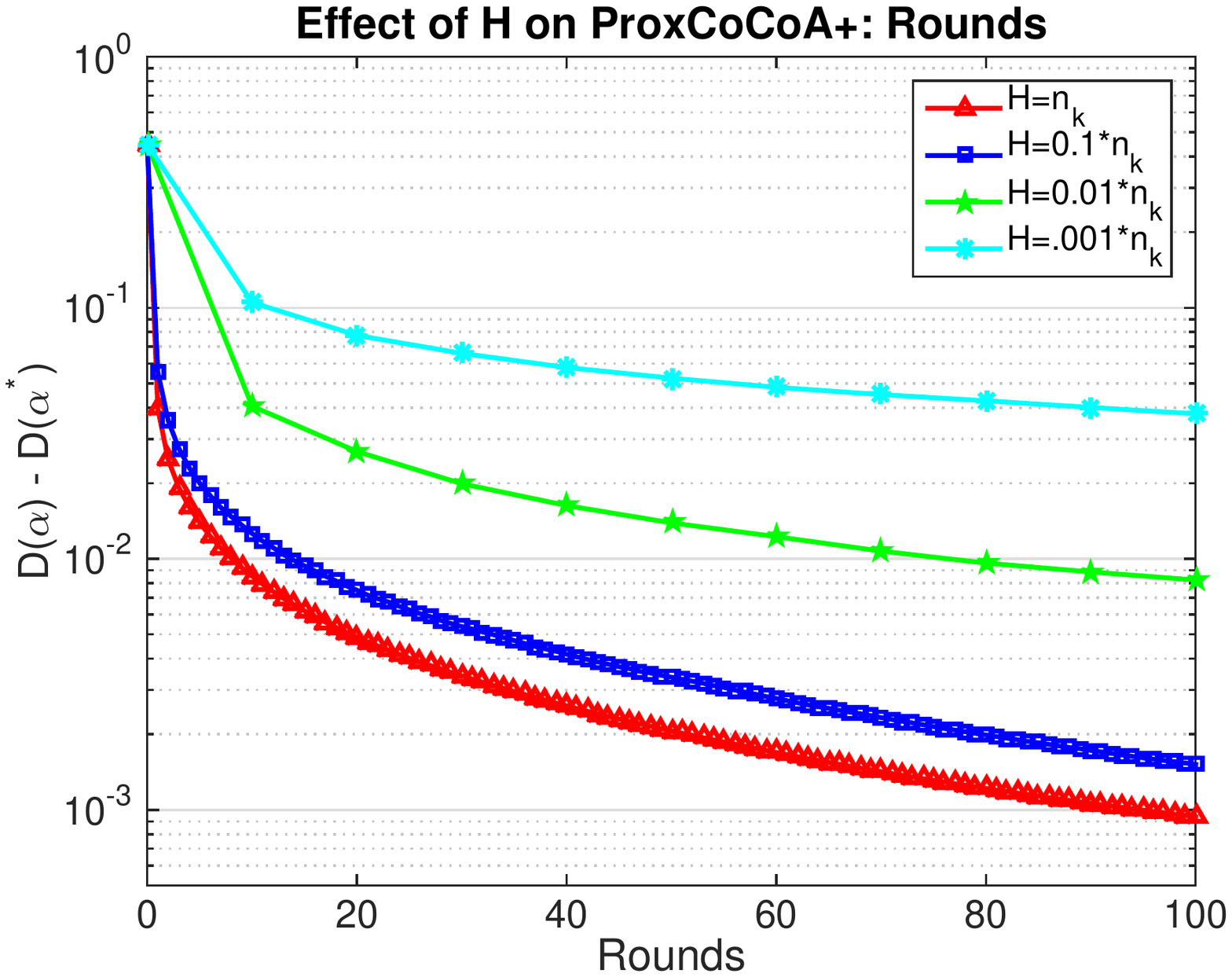}
\tinytrimfig{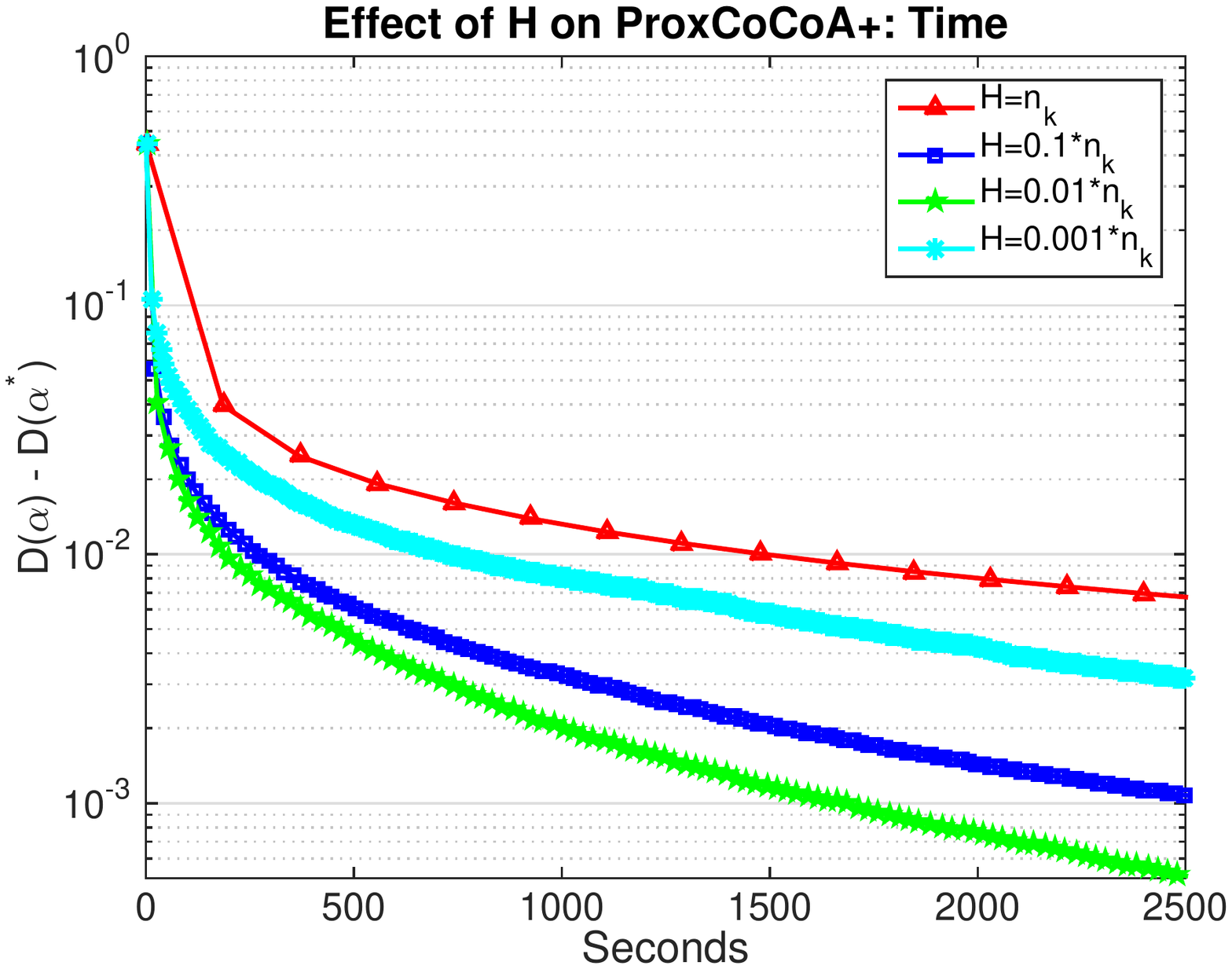}
\vspace{-1.5em}
\caption{\small{Suboptimality in terms of ${\bD}(\alphav)$ for solving Lasso for the webspam dataset ($K$=16, $\lambda$=\mbox{1\sc{e}-5}). Here we illustrate how the work spent in the local subproblem (given by $H$) influences the total performance of \proxcocoa in terms of number of rounds as well as wall time.}} 
\label{fig:heffect}
\end{minipage}

\vspace{2mm}

Results are shown for regularization parameters $\lambda$ such that the resulting weight vector $\alphav$ is sparse. However, our results are robust to varying values of~$\lambda$ as well as to various problem settings, as we illustrate in Figure~\ref{fig:lambda}.

\newcommand{\newtinytrimfig}[1]{\subfigure{\includegraphics[trim = 45 180 60 180, clip, width=.49\linewidth]{#1}}}
\setlength{\columnsep}{13pt}
\begin{wrapfigure}{r}{6.8cm}
\vspace{-2em}
\begin{minipage}[h]{.48\textwidth}
\centering
\captionsetup{type=figure}
\newtinytrimfig{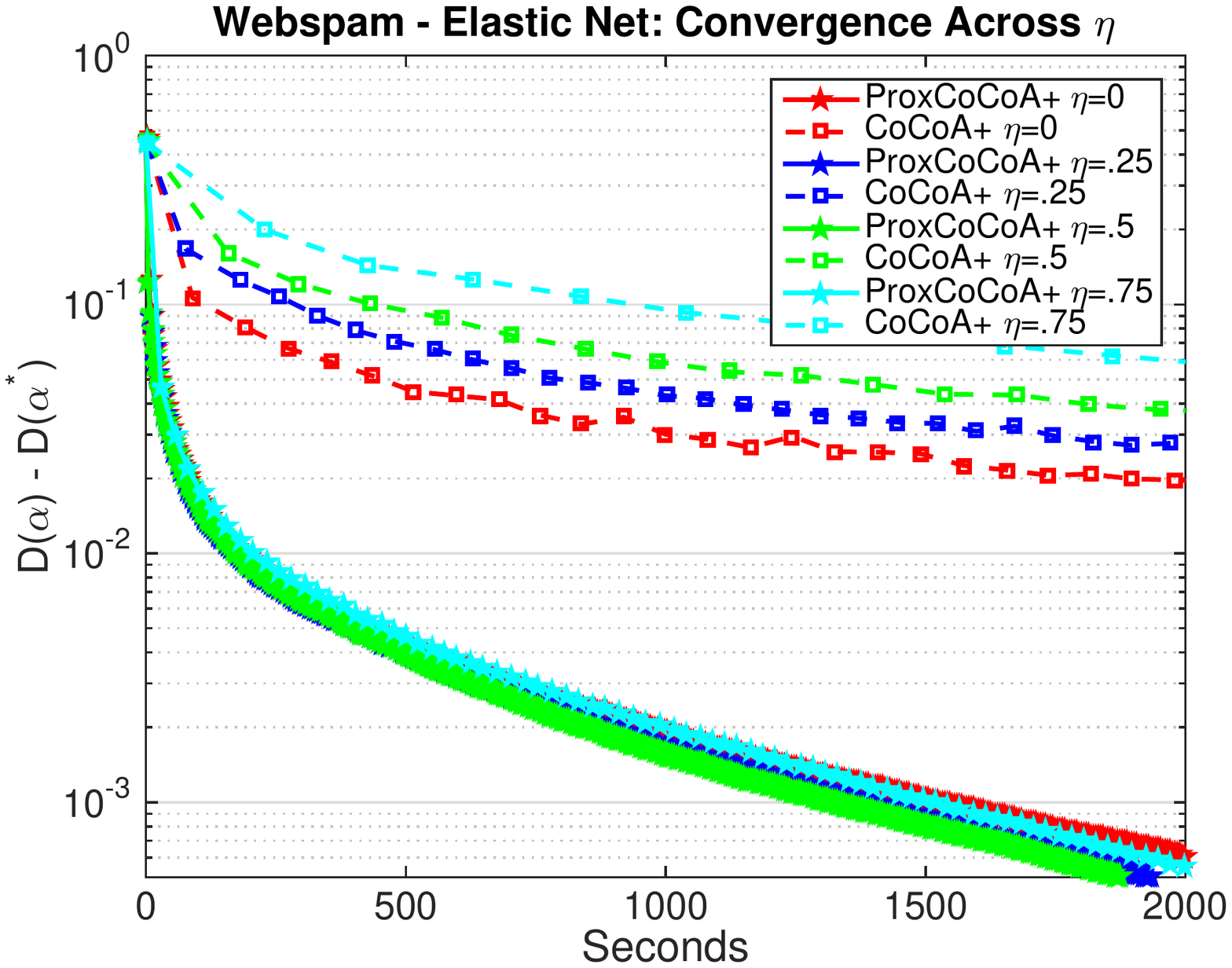} 
\newtinytrimfig{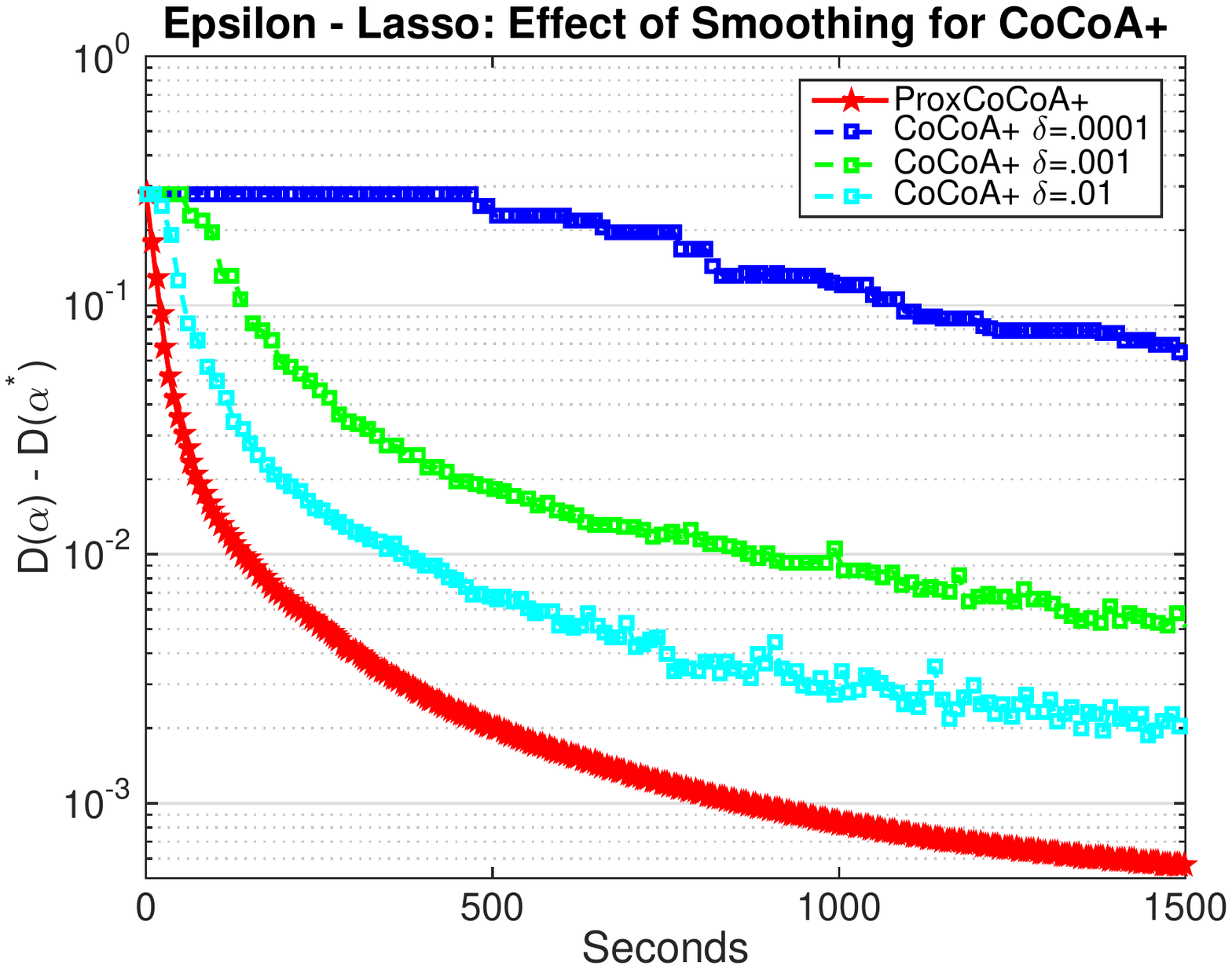} \\
\vspace{-.5em}
\hspace{.6em} {\small(4a)} \hspace{8em} {\small(4b)} \\

\captionof*{table}{\small{Sparsity of final iterates in (4b)} \vspace{-1mm}}
\label{tab:sparsity}
   \begin{center}
      \begin{tabular}{| c | c |  
      c | c |}
      \hline
    {\tiny\textbf{\proxcocoa}} & {\tiny\textbf{$\delta=0.0001$}} & 
    {\tiny\textbf{$\delta=0.001$}} & {\tiny\textbf{$\delta=0.01$}}   \\
    \hline
	\tiny{0.6030} & \tiny{0.6035} & \tiny{0.6240} & \tiny{0.6465} \\
	\hline
      \end{tabular}
   \end{center} 
   {\small(4c)}
   
\end{minipage}
\vspace{-.35em} 
\caption{\small{\cocoap is ill-equipped to deal with large feature sizes as compared to \proxcocoa, and the gap in performance increases as the regularizer becomes less strongly convex, e.g., as $\eta \to 1.0$ for elastic net (4a). For pure $L_1$ regularization, Nesterov smoothing is not an effective option for \cocoap. It either modifies the solution (4c) or slows convergence (4b).}}
\label{fig:cocoa}
\vspace{-2em}
\end{wrapfigure}

We note that in contrast to the compared methods, \proxcocoa comes with the benefit of having only a single parameter to tune: the subproblem approximation quality, $\Theta$, which can be controlled via the number of local subproblem iterations, $H$. We further explore the effect of this parameter in Figure~\ref{fig:heffect}, and provide a general guideline for choosing it in practice (see Remark~\ref{rem:localtime}). 
In particular, we see that while increasing $H$ always results in better performance in terms of rounds, smaller or larger values of $H$ may result in better performance in terms of wall-clock time, depending on the cost of communication and computation. The flexibility to tune $H$ is one of the reasons for \proxcocoa's significant performance gains.

\paragraph{Comparison with \cocoap.}

Finally, we point out several important ways in which \proxcocoa improves upon the \cocoap framework \cite{Ma:2015ti}. First, \cocoap cannot be included in the set of experiments in Figure~\ref{fig:comparison} because it cannot be directly applied to the Lasso objective (\cocoap only allows for strongly convex regularizers\footnote{\cocoap in \cite{Ma:2015ti} is in fact limited to the case where the regularizer is equal to the $L_2$ norm $\frac{1}{2}\|\cdot\|_2^2$, though the extension to strongly convex regularizers is covered as a special case in our analysis.}). Second, as shown in Figure~\ref{fig:cocoa}, the performance of \cocoap degrades drastically when considering datasets with large numbers of features, such as the webspam dataset. One reason for this is that \cocoap distributes data by data point, which necessitates communicating a vector of length equal to the feature size.
When the feature size is large, this can become expensive. The results shown hold despite the fact that we have tuned $H$ (the number of local solver iterations) separately for both \proxcocoa and \cocoap. 

Beyond communication, we also see that \cocoap is slower to converge as the regularizer becomes less strongly convex (Figure~\ref{fig:cocoa}a). Indeed, even when the number of features is relatively low such as for the epsilon dataset, we see that the performance of \cocoap degrades significantly as the regularizer approaches pure $\Lone$. In Figure~\ref{fig:cocoa}, we illustrate this by implementing the Nesterov smoothing technique used in, e.g., \cite{ShalevShwartz:2014dy,Zhang:2015vj} --- adding a small amount of strong convexity $\delta \|\alphav\|_2^2$ to the objective for Lasso regression. We show results for decreasing levels of $\delta$. As $\delta$ decreases, the final sparsity of the problem starts to match that of running pure $L_1$ (Figure~\ref{fig:cocoa}c), but the performance also degrades (Figure~\ref{fig:cocoa}b). We note again that through the modification presented in Section~\ref{sec:convergence}, we can deliver strong rates without having to make these fundamental alterations to the problem of interest.

\newpage
\section*{Acknowledgments} 
We thank Michael P. Friedlander and Martin Tak{\'a}{\v c} for fruitful discussions.
{\small
\bibliographystyle{abbrv}
\bibliography{bibliography}

\begin{thebibliography}{10}

\bibitem{Andrew:2007cu}
G.~Andrew and J.~Gao.
\newblock {Scalable training of L1-regularized log-linear models}.
\newblock In {\em ICML}, 2007.

\bibitem{Bauschke:2011ik}
H.~H. Bauschke and P.~L. Combettes.
\newblock {\em {Convex Analysis and Monotone Operator Theory in Hilbert
  Spaces}}.
\newblock CMS Books in Mathematics. Springer New York, New York, NY, 2011.

\bibitem{Bian:2013wx}
Y.~Bian et~al.
\newblock Parallel coordinate descent newton method for efficient
  $\ell$$_{1}$-regularized minimization.
\newblock {\em arXiv.org}, 2013.

\bibitem{Borwein:2005ge}
J.~M. Borwein and Q.~Zhu.
\newblock {\em {Techniques of Variational Analysis and Nonlinear
  Optimization}}.
\newblock Canadian Mathematical Society Books in Math, Springer New York, 2005.

\bibitem{Boyd:2010bw}
S.~Boyd et~al.
\newblock Distributed optimization and statistical learning via the alternating
  direction method of multipliers.
\newblock {\em Foundations and Trends in Machine Learning}, 3(1):1--122, 2010.

\bibitem{Boyd:2004uz}
S.~Boyd and L.~Vandenberghe.
\newblock {\em {Convex Optimization}}.
\newblock Cambridge University Press, 2004.

\bibitem{Bradley:2011wq}
J.~K. Bradley et~al.
\newblock {Parallel coordinate descent for l1-regularized loss minimization}.
\newblock In {\em ICML}, 2011.

\bibitem{Dunner:2016vga}
C.~D{\"u}nner et~al.
\newblock {Primal-Dual Rates and Certificates}.
\newblock In {\em ICML}, 2016.

\bibitem{Fercoq:2015kd}
O.~Fercoq and P.~Richt{\'a}rik.
\newblock {Accelerated, Parallel, and Proximal Coordinate Descent}.
\newblock {\em SIAM Journal on Optimization}, 25(4):1997--2023, Oct. 2015.

\bibitem{Forte:2015wv}
S.~Forte.
\newblock {Distributed Optimization for Non-Strongly Convex Regularizers}.
\newblock Master's thesis, ETH Z{\"u}rich, Sept. 2015.

\bibitem{Friedman:2010wm}
J.~Friedman, T.~Hastie, and R.~Tibshirani.
\newblock {Regularization paths for generalized linear models via coordinate
  descent}.
\newblock {\em Journal of Statistical Software}, 33(1):1--22, 2010.

\bibitem{Jaggi:2014vi}
M.~Jaggi et~al.
\newblock {Communication-efficient distributed dual coordinate ascent}.
\newblock In {\em NIPS}, 2014.

\bibitem{Johnson:2015tq}
T.~Johnson and C.~Guestrin.
\newblock {Blitz: A Principled Meta-Algorithm for Scaling Sparse Optimization}.
\newblock In {\em ICML}, 2015.

\bibitem{Kakade:2009wh}
S.~M. Kakade, S.~Shalev-Shwartz, and A.~Tewari.
\newblock {On the duality of strong convexity and strong smoothness: Learning
  applications and matrix regularization}.
\newblock Technical report, TTI, 2009.

\bibitem{Kang:2014aa}
Kang et~al.
\newblock {Data/feature distributed stochastic coordinate descent for logistic
  regression}.
\newblock In {\em CIKM}, 2014.

\bibitem{Lu:2013tl}
Z.~Lu and L.~Xiao.
\newblock On the complexity analysis of randomized block-coordinate descent
  methods.
\newblock {\em arXiv.org}, 2013.

\bibitem{Ma:2015ti}
C.~Ma et~al.
\newblock Adding vs. averaging in distributed primal-dual optimization.
\newblock In {\em ICML}, 2015.

\bibitem{Mahajan:2014tg}
D.~Mahajan, S.~S. Keerthi, and S.~Sundararajan.
\newblock A distributed block coordinate descent method for training $l_1$
  regularized linear classifiers.
\newblock {\em arXiv.org}, 2014.

\bibitem{McMahan:2013ux}
H.~B. McMahan et~al.
\newblock Ad click prediction: a view from the trenches.
\newblock In {\em KDD}, 2013.

\bibitem{Meng:2015tu}
X.~Meng et~al.
\newblock Mllib: Machine learning in apache spark.
\newblock {\em arXiv.org}, 2015.

\bibitem{Necoara:2016cr}
I.~Necoara and D.~Clipici.
\newblock {Parallel Random Coordinate Descent Method for Composite
  Minimization: Convergence Analysis and Error Bounds}.
\newblock {\em SIAM Journal on Optimization}, 26(1):197--226, Jan. 2016.

\bibitem{Nesterov:2005ic}
Y.~Nesterov.
\newblock {Smooth minimization of non-smooth functions}.
\newblock {\em Mathematical Programming}, 103(1):127--152, 2005.

\bibitem{Qu:2015ve}
Z.~Qu et~al.
\newblock {SDNA}: Stochastic dual newton ascent for empirical risk
  minimization.
\newblock {\em arXiv.org}, 2015.

\bibitem{Rockafellar:1997ww}
R.~T. Rockafellar.
\newblock {\em {Convex Analysis}}.
\newblock Princeton University Press, 1997.

\bibitem{ShalevShwartz:2011vo}
S.~Shalev-Shwartz and A.~Tewari.
\newblock {Stochastic methods for l$_{1}$-regularized loss minimization}.
\newblock {\em Journal of Machine Learning Research}, 12:1865--1892, 2011.

\bibitem{ShalevShwartz:2013wl}
S.~Shalev-Shwartz and T.~Zhang.
\newblock Stochastic dual coordinate ascent methods for regularized loss
  minimization.
\newblock {\em Journal of Machine Learning Research}, 14:567--599, 2013.

\bibitem{ShalevShwartz:2014dy}
S.~Shalev-Shwartz and T.~Zhang.
\newblock {Accelerated proximal stochastic dual coordinate ascent for
  regularized loss minimization}.
\newblock {\em Mathematical Programming}, Series A:1--41, 2014.

\bibitem{Tappenden:2015vha}
R.~Tappenden and P.~Richt{\'a}rik.
\newblock On the complexity of parallel coordinate descent.
\newblock {\em arXiv.org}, 2015.

\bibitem{Trofimov:2014vb}
I.~Trofimov and A.~Genkin.
\newblock Distributed coordinate descent for l1-regularized logistic
  regression.
\newblock {\em arXiv.org}, 2014.

\bibitem{Yang:2013vl}
T.~Yang.
\newblock Trading computation for communication: Distributed stochastic dual
  coordinate ascent.
\newblock In {\em NIPS}, 2013.

\bibitem{Yen:2015vy}
I.~E.-H. Yen, S.-W. Lin, and S.-D. Lin.
\newblock {A Dual Augmented Block Minimization Framework for Learning with
  Limited Memory}.
\newblock In {\em NIPS}, 2015.

\bibitem{Yu:2010vw}
J.~Yu, S.~Vishwanathan, S.~G{\"u}nter, and N.~N. Schraudolph.
\newblock A quasi-newton approach to nonsmooth convex optimization problems in
  machine learning.
\newblock {\em Journal of Machine Learning Research}, 2010.

\bibitem{Yuan:2010ub}
G.-X. Yuan et~al.
\newblock A comparison of optimization methods and software for large-scale
  l1-regularized linear classification.
\newblock {\em Journal of Machine Learning Research}, 2010.

\bibitem{Yuan:2012wi}
G.-X. Yuan, C.-H. Ho, and C.-J. Lin.
\newblock An improved glmnet for l1-regularized logistic regression.
\newblock {\em Journal of Machine Learning Research}, 2012.

\bibitem{Zhang:2015vj}
Y.~Zhang and X.~Lin.
\newblock {Stochastic Primal-Dual Coordinate Method for Regularized Empirical
  Risk Minimization}.
\newblock In {\em ICML}, 2015.

\bibitem{Zheng16}
S.~Zheng et~al.
\newblock A general distributed dual coordinate optimization framework for
  regularized loss minimization.
\newblock In {\em arXiv.org}, 2016.

\end{thebibliography}
}

\newcommand{\citesup}{\cite} 
\newcommand{\citepsup}{\cite} 

\clearpage
\appendix
\onecolumn
\part*{Appendix}

\section{Definitions}

\begin{definition}[$L$-Lipschitz Continuity]
A function $f: \R^d \to \R$ is \emph{$L$-Lipschitz continuous} if $\forall a,b \in \R^d$, we have
\begin{equation}
 | f(a) - f(b) | \leq L \| a-b \| \, .
\end{equation}
\end{definition}

\begin{repdefinition}{def:lbounded}[$L$-Bounded Support]
\upshape A function $f: \R^d \to \R$ has \emph{$L$-bounded support} if its effective domain is bounded by $L$, i.e.,
\begin{equation}
  f(\uv) < + \infty  \ \Rightarrow \  \|\uv\| \le L \, .
\end{equation}
\end{repdefinition}

\begin{definition}[$L$-Smoothness]
A function $f:\R^d\rightarrow\R$ is called \emph{$L$-smooth}, for $L>0$, if
it is differentiable and its derivative is $L$-Lipschitz continuous,
or equivalently
\begin{equation}
f(\uv) \leq f(\wv) + \langle \nabla f(\wv), \uv-\wv \rangle + \frac{L}{2} \| \uv-\wv \|^2  \qquad\forall \uv,\wv\in\R^d \, .
\label{eq:smooth}
\end{equation}
\end{definition}

\begin{definition}[$\mu$-Strong Convexity]
A function $f:\R^d\rightarrow\R$ is called \emph{$\mu$-strongly convex}, for $\mu\ge0$, if
\begin{equation}
f(\uv) \geq f(\wv) + \langle \nabla f(\wv), \uv-\wv \rangle + \frac{\mu}{2} \| \uv-\wv \|^2  \qquad\forall \uv,\wv\in\R^d \, .
\label{eq:strongconv}
\end{equation}
And analogously if the same holds for all subgradients, in the case of a general closed convex function~$f$.
\end{definition}

\section{Convex Conjugates}\label{sec:conjugates}

The convex conjugate of a function $f: \R^d\rightarrow \R$ is defined as 
\begin{equation}
f^*(\vv) := \max_{\uv\in\R^d} \vv^T \uv - f(\uv) \, .
\end{equation}
Below we list several useful properties of conjugates (see, e.g., \citepsup[Section 3.3.2]{Boyd:2004uz}):
\vspace{-1mm}
\begin{itemize} 
\addtolength{\itemindent}{.25em}
\item Double conjugate: \hspace{1em}
$(f^*)^* = f$ if $f$ is closed and convex.
\item Value Scaling:  (for $\alpha>0$) \hspace{2em}
$
f(\vv) = \alpha g(\vv) 
\qquad\Rightarrow\qquad
f^*(\wv) = \alpha g^*(\wv/\alpha)    \, .
$
\item Argument Scaling:  (for $\alpha\ne0$) \hspace{1em}
$
f(\vv) = g(\alpha \vv) 
\qquad\Rightarrow\qquad
f^*(\wv) = g^*(\wv/\alpha) \, .
$
\item Conjugate of a separable sum: \hspace{1em}
$
f(\vv)=\sum_i \phi_i(v_i)
\qquad\Rightarrow\qquad
f^*( \wv ) = \sum_i \phi_i^* ( w_i ) \, .
$
\end{itemize}\vspace{1mm}

\begin{lemma}[{Duality between Lipschitzness and L-Bounded Support, \citepsup[Corollary 13.3.3]{Rockafellar:1997ww}}]
\label{lem:dualLipschitz}
Given a proper convex function $f$, it holds that
$f$ is $L$-Lipschitz
if and only if 
$f^*$ has $L$-bounded support.
\end{lemma}

\begin{lemma}[{Duality between Smoothness and Strong Convexity, \citepsup[Theorem 6]{Kakade:2009wh}}]
\label{lem:dualSmooth}
Given a closed convex function~$f$, it holds that
$f$ is $\mu$ strongly convex w.r.t. the norm $\|\cdot\|$
if and only if
$f^*$ is $(1/{\mu})$-smooth w.r.t. the dual norm~$\|\cdot\|_*$.
\end{lemma}

\section{Applications}
\label{sec:applications}

\subsection{$\Lone$ and General Non-Strongly Convex Regularizers}
\label{sec:l1}

$\Lone$ regularization is obtained in the objective~\eqref{eq:obj} by letting $\ell_i(\cdot) := \lambda |\cdot |$.
Primal-dual convergence can be obtained by using the modification introduced in Section~\ref{sec:convergence}, which will guarantee $L$-bounded support. Formally, we replace $\ell_i(\cdot) = |\cdot|$ by 
\[
\bar{\ell}(\alpha) :=
\begin{cases}
|\alpha| &: \alpha \in [-B,B] \\
+\infty &: \text{otherwise.}
\end{cases}
\]
For large enough $B$, this problem yields the same solution as the original $\Lone$-objective. We provide a detailed proof and description of this technique in Section~\ref{sec:appendix-conjugates}.
Note that this only affects convergence theory, in that it allows us to present a strong primal-dual rate (Theorem~\ref{thm:convergenceNonsmooth} for $L$=$B$).

\subsection{Elastic Net and General Strongly Convex Regularizers} 
\label{sec:elasticnet}
Another application we can consider is elastic net regularization, $
  \frac{\eta}{2} \norm{\alphav}_2^2 +  (1-\eta) \norm{\alphav}_1$,
for fixed parameter $\eta \in (0,1]$, which is obtained by setting $\ell_i(\alpha) := \lambda\big[ \frac{\eta}{2} \alpha^2 + (1-\eta) |\alpha| \big]$ in~\eqref{eq:obj}. For the special case $\eta=0$, we obtain the $\Lone$-norm. 
For elastic-net-regularized problems of the form \eqref{eq:obj}, Theorem~\ref{thm:convergenceSmooth} gives a global linear (geometric) convergence rate, since $\ell_i$ is $\eta$-strongly convex. This holds as long as the data-fit function is smooth (see Section \ref{sec:datafit}), and directly yields a primal-dual algorithm and corresponding rate.

\vspace{-1mm}
\subsection{Local Solvers for $\Lone$ and Elastic Net}
For the $\Lone$-regularizer in the primal setting, the local subproblem \eqref{eq:subproblem} becomes a simple quadratic problem on the local data, with regularization applied only to local variables $\vsubset{\alphav}{k}$. Therefore, existing fast $\Lone$-solvers for the single-machine case, such as \glmnet variants~\cite{Friedman:2010wm} or \textsc{blitz}~\cite{Johnson:2015tq} can be directly applied to each local subproblem $\Ggk(\,\cdot\,; \vv, \vsubset{\alphav}{k})$ within Algorithm~\ref{alg:generalizedcocoa}. The sparsity induced on the subproblem solutions of each machine naturally translates into the sparsity of the global solution, since the local variables $\vsubset{\alphav}{k}$ will be concatenated.

In terms of the approximation quality parameter $\Theta$ for the local problems (Assumption~\ref{asm:theta}), we can apply existing recent convergence results from the single machine case. For example, for randomized coordinate descent (as part of \glmnet), \citepsup[Theorem 1]{Lu:2013tl} gives a $O(1/t)$ approximation quality for any separable regularizer, including $\Lone$ and elastic net; see also~\cite{Tappenden:2015vha,ShalevShwartz:2011vo}.

\subsection{Smooth Data-Fit Functions}\label{sec:datafit}

To illustrate the role of $f$ as a smooth data-fit function in this section---contrasting with its role as a regularizer in traditional \cocoap as we discuss in Section \ref{sec:oldcocoa}---we consider the following examples.

\vspace{-.5em}
\paragraph{Least squares loss.}
Let $\bv \in \R^d$ be labels or response values, and consider the least squares objective $f(\vv) := \frac12 \|\vv - \bv\|_2^2$, which is $1$-smooth. We obtain the familiar least-squares regression objective in our optimization problem \eqref{eq:obj}, using
\[
f(A\alphav) := \textstyle\frac{1}{2}\| A\alphav - \bv\|_2^2 \, .
\] 
Observing that the gradient of $f$ is $\nabla f(\vv)=\vv-\bv$, the dual-to-primal mapping is given by: $\wv(\alphav)$ $:=$ $\nabla f( \vv(\alphav) )$ $=$ $A \alphav -\bv$, which is well known as the \textit{residual vector} in least-squares regression.

\vspace{-.5em}
\paragraph{Logistic regression loss.}
For classification problems, we consider a logistic regression model with $d$ training examples $\yv_j\in\R^n$ for $j\in[d]$ collected as the rows of the data matrix $A$.
For each training example, we are given a binary label, which we collect in the vector $\bv \in \{-1,1\}^d$. 
Formally, the objective is defined as $f(\vv) := \sum_{j=1}^d \log{(1 + \exp{(-b_j v_j)})}$, which is again a separable function.
The classifier loss is given by
\vspace{-3mm}
\begin{equation}
\label{eq:logisticlossprimal}
f(A\alphav) := \sum_{j=1}^d \log{(1 + \exp{(-b_j \yv_j^T \alphav)})} \, , 
\end{equation}
where $\alphav\in\R^n$ is the parameter vector. It is not hard to show that $f$ is $1$-smooth if the labels satisfy $b_j\in[-1,1]$; see e.g. Lemma~\ref{lem:logisticlossconj} below.
The primal-dual mapping $\wv(\alphav) := \nabla f( \vv(\alphav) ) =\nabla f( A \alphav )$ is given by
$
w_j(\alphav) 
= \frac{-b_j}{1 + \exp{(b_j \yv_j^T \alphav)}} \, .$  

\section{Proofs of Primal-Dual Relationship}

In the following subsections we provide derivations of the primal-dual relationship of the general objectives~\eqref{eq:obj} and \eqref{eq:dualP}, and then show how to derive this primal-dual setup for various applications.

\subsection{Primal-Dual Relationship}
The relation of our original formulation \eqref{eq:obj} to its dual formulation \eqref{eq:dualP} is standard in convex analysis, and is a special case of the concept of Fenchel Duality. 
Using the combination with the linear map $A$ as in our case, the relationship is called Fenchel-Rockafellar Duality, see e.g. \citep[Theorem 4.4.2]{Borwein:2005ge} or \citepsup[Proposition 15.18]{Bauschke:2011ik}. For completeness, we illustrate this correspondence with a self-contained derivation of the duality.

Starting with the original formulation \eqref{eq:obj}, we introduce a helper variable vector $\vv\in \R^d$ representing $\vv =A\alphav$. Then optimization problem~\eqref{eq:obj} becomes:
\begin{equation}
\label{eq:constrainedprimal}
\min_{\alphav \in\R^n} \quad  f(\vv) + g( \alphav) \quad \text{such that} \ \vv =A\alphav \, .
\end{equation}
Introducing Lagrange multipliers $\wv \in \R^d$,  the Lagrangian is given by:
\[
L(\alphav, \vv; \wv) := f(\vv) +   g(\alphav) + \wv^T\left(A\alphav-\vv\right) \, .
\]
The dual problem of \eqref{eq:obj} follows by taking the infimum with respect to both $\alphav$ and $\vv$:
\begin{align}
\inf_{\alphav, \vv} L(\wv, \alphav, \vv) & =   \inf_{\vv} \left\{ f(\vv) - \wv^T \vv \right\} + \inf_{\alphav} \left\{ g(\alphav) +  \wv^T A\alphav\right\} \notag \\
& =  - \sup_{\vv} \left\{  \wv^T \vv - f(\vv) \right\}- \sup_{\alphav} \left\{(-\wv^T A)\alphav -  g(\alphav) \right\} \notag\\
& = - f^*(\wv) - g^*(-A^T \wv) \label{eq:Lagrangian}\, .
\end{align}
We change signs and turn the  maximization of the dual problem \eqref{eq:Lagrangian} into a minimization and thus we arrive at the dual formulation $\eqref{eq:dualP}$ as claimed:
\[
    \min_{\wv \in \R^d} \quad \Big[ \ 
    \bP(\wv) := f^*(\wv) + g^*(-A^T \wv) \ \Big] \, .
\]

\subsection{Conjugates and Smoothness of $f$-Functions of Interest}

\begin{lemma}[Conjugate and Smoothness of the Logistic Loss]
\label{lem:logisticlossconj}
The logistic classifier loss function
\begin{equation*}
f(A\alphav) := \sum_{j=1}^d \log{(1 + \exp{(-b_j \yv_j^T \alphav)})} 
\vspace{-1mm}
\end{equation*}
(see also \eqref{eq:logisticlossprimal} above) is the conjugate of $f^*$, where
\begin{equation}
\label{eq:logisticlossconj}
f^*(\wv) := \sum_{j=1}^d \big((1+w_j b_j) \log{(1+w_j b_j)} - w_j b_j\log{(-w_j b_j)\big)} \, ,
\end{equation}
with the box constraint $-w_jb_j \in [0,1]$.

Furthermore, $f^*(\wv)$ is $1$-strongly convex over its domain if the labels satisfy $b_j\in[-1,1]$.
\end{lemma}
\begin{proof}[Proof of Lemma \ref{lem:logisticlossconj}]
By separability of $f^*$, the conjugate of $f^*(\vv) = \sum_j \phi^*_j(v_j)$ is $f(\wv) = \sum_j \phi_j(w_j)$. For the losses, the conjugate pairs are $\phi_j(u) = \log(1+\exp(-b_ju))$, and $\phi^*_j (w_j) = -w_jb_j \log(-w_jb_j)+ (1+w_jb_j)\log(1+w_jb_j)$ with $-w_jb_j \in [0,1]$, see e.g. \citep[Page 577]{ShalevShwartz:2013wl}.

For the strong convexity, we show 1-strong smoothness of the conjugate $f(\vv) := \sum_{j=1}^d \log{(1 + \exp{(-b_j v_j)})} = \sum_{j=1}^d h(b_j v_j)$, which is an equivalent property, see Lemma \ref{lem:dualSmooth}.
Using the second derivative $h''(a)=\frac{e^{-a}}{(1+e^{-a})^2} \le 1$ of the function $h(a)=\log(1+e^{-a})$, we have that 
$\nabla^2 f(\vv) 
= \diag\big( (h''(b_j v_j) b_j^2)_j \big) 
= \diag\big( (\frac{e^{-b_j v_j}}{(1+e^{-b_j v_j})^2} b_j^2)_j \big) \le 1$, so $f(\vv)$ is 1-smooth w.r.t. the Euclidean norm.
\end{proof}

\subsection{Conjugates of Common Regularizers}
\label{sec:appendix-conjugates}

\begin{lemma}[Conjugate of the Elastic Net Regularizer]
\label{lem:elasticnetconjugate}
For $\eta \in (0,1]$, the elastic net function $\ell_i(\alpha) := \frac{\eta}{2} \alpha^2 + (1-\eta) |\alpha|$ is the convex conjugate of 
\[
    \ell^*_i(x) := \textstyle\frac1{2\eta} \big(\big[|x|-(1-\eta)\big]_+\big)^2 ,
\]
where $[.]_+$ is the positive part operator, $[s]_+ = s$ for $s>0$, and zero otherwise.
Furthermore, this $\ell^*$ is smooth, i.e. has Lipschitz continuous gradient with constant $1/\eta$.
\end{lemma}
\begin{proof}
We start by applying the definition of convex conjugate, that is:
\[\textstyle 
\ell^*(x) = \max_{\alpha \in \R} \left[ x\alpha - \eta \frac{\alpha^2}{2} - (1-\eta)|\alpha| \right] \, .
\]

We now distinguish two cases for the optimal: $\alpha^{\star} \geq 0$, $\alpha^{\star} < 0$.
For the first case we get that
\[\textstyle 
\ell^*(x) = \max_{\alpha \in \R} \left[ x\alpha - \eta \frac{\alpha^2}{2} - (1-\eta)\alpha \right] \, .
\]
Setting the derivative to $0$ we get $\alpha^{\star} = \frac{x-(1-\eta)}{\eta}$. To satisfy $\alpha^{\star} \geq 0$, we must have $x \geq 1-\eta$.
Replacing with $\alpha^{\star}$ we thus get:
\[\textstyle 
\ell^*(x) = \alpha^{\star} (x - \frac{1}{2}\eta \alpha^{\star} -(1-\eta)) = \alpha^{\star} \left( x - \frac{1}{2} (x-(1-\eta)) - (1-\eta) \right) =
\]
\[\textstyle 
\frac{1}{2} \alpha^{\star} \left( x - (1-\eta) \right) = \frac{1}{2\eta} \left( x - (1-\eta) \right)^2 \, .
\]
Similarly we can show that for $x \leq -(1-\eta)$
\[\textstyle 
\ell^*(x) = \frac{1}{2\eta} \left( x + (1-\eta) \right)^2 \, .
\]
Finally, by the fact that $\ell^*(.)$ is convex, always positive, and $\ell^*(-(1-\eta)) = \ell^*(1-\eta) = 0$,
it follows that $\ell^*(x) = 0$ for every $x \in [-(1-\eta),1-\eta]$.

For the smoothness properties, we consider the derivative of this function $\ell^*(x)$ and see that $\ell^*(x)$ is smooth, i.e. has Lipschitz continuous gradient with constant $1/\eta$, assuming $\eta>0$.
\end{proof}

\paragraph{Continuous conjugate modification for indicator functions.} To apply the theoretical convergence result from Theorem \ref{thm:convergenceNonsmooth} to objectives with $L_1$ norms, we modify the function $|\cdot|$ by imposing an additional constraint. Consider replacing $\ell_i(\cdot) = |\cdot|$ by
\[
\bar{\ell}(\alpha) :=
\begin{cases}
|\alpha| &: \alpha \in [-B,B] \\
+\infty &: \text{otherwise.}
\end{cases}
\]
With this modified $\Lone$-regularizer, the optimization problem \eqref{eq:obj} with regularization parameter $\lambda$ becomes
\begin{equation}
\label{eq:surrogatel1regproblem}
\min_{\alphav \in \R^{n}} \ f(A \alphav)  + \lambda \sum_{i = 1}^n \bar{\ell}(\alpha_i) \, .
\end{equation}
For large enough choice of the value $B$, this problems yields the same solution as the original objective:

\vspace{-2em}
\begin{equation}
    \label{eq:scaledl1}
    \min_{\alphav \in \R^{n}} \quad \Big[ \
    \bD(\alphav) := \ f(A \alphav)  + \lambda \sum_{i=1}^{n} |\alpha_i| \ \Big] \, .
\end{equation}

As we can see, the $\bar{\ell}$ is nothing more than a constrained version of the absolute value to the interval $[-B,B]$.
Therefore by setting $B$ to a large enough value that the interesting values of $\alpha_i$ will never reach,
we can have continuous $\bar{\ell}^*$ and at the same time make \eqref{eq:surrogatel1regproblem} equivalent to \eqref{eq:scaledl1}.

Formally, a simple way to obtain a large enough value of $B$, so that all solutions of \eqref{eq:scaledl1} are unaffected is the following. Note that we start the algorithm at $\alphav = \0$.
For every solution encountered during the execution of the algorithm, the objective values should never become worse than ${\bD}(\0)$. In other words, we restrict the ${\bD}(\cdot)$ optimization problem to the level set given by the initial starting value. Formally, this means that for every $i$, we will always require:
\[
\lambda |\alpha_i| \leq {f(\0)=\bD(\0)}  \implies |\alpha_i| \leq \frac{f(\0)}{\lambda}.
\]
(Note that $f(\alphav) \ge 0$ holds without loss of generality). We can thus set the value of $B$ to be $\frac{f(\0)}{\lambda }$.

\vspace{1em}
\begin{lemma}[Conjugate of the modified $\Lone$-norm]
\label{lem:l1surrogate}
The convex conjugate of $\bar{\ell}_i$ as defined above is
\[
    \bar{\ell}^*(x) = 
    \begin{cases}
            0 & : x \in [-1,1]  \\
            B(|x| - 1) & : \text{otherwise,}
        \end{cases}
\]
and is $B$-Lipschitz.
\end{lemma}

\begin{proof}
We start by applying the definition of convex conjugate:
\[
\bar{\ell}(\alpha) = \sup_{x \in \R} \left[ \alpha x - \bar{\ell}^*(x) \right] \, .
\]
We begin by looking at the case in which $\alpha \geq B$; in this case it's easy to see that when $x \to +\infty$, we
have:
\[
\alpha x - B(|x|-1) = (\alpha - B)x - B \to +\infty
\]
as $\alpha - B \geq 0$. The case $\alpha \leq -B$ holds analogously.
We'll now look at the case $\alpha \in [0,B]$; in this case it is clear we must have $x^{\star} \ge 0$.
 It also must hold that $x^{\star} \leq 1$, since
\[
\alpha x - B(x-1) < \alpha x
\]
for every $x > 1$. Therefore the maximization becomes

\[
\bar{\ell}(\alpha) = \sup_{x \in [0,1]} \alpha x \, ,
\]
which has maximum $\alpha$ at $x = 1$.  The remaining $\alpha \in [-B,0]$ case follows in similar fashion.
\end{proof}

\section{Convergence Proofs}
\label{sec:convergenceproofs}
In this section we provide proofs of our main convergence results. The results are motivated by \cite{Ma:2015ti}, but where we have significantly generalized the problem of interest, and where we derive separate meaning by applying the problem directly to~\eqref{eq:obj}. We provide full details of Lemma~\ref{lem:RelationOfDTOSubproblems} as a proof of concept, but omit details in later proofs that can be derived using the arguments in \cite{Ma:2015ti} or earlier work of \cite{ShalevShwartz:2013wl}, and instead outline the proof strategy and highlight sections where the theory deviates.

\subsection{Approximation of $\bD(\cdot)$ by the Local Subproblems $\Ggk(\cdot)$}
We begin with a definition of the data-dependent aggregation parameter for \proxcocoa, $\sigma'$, which we will use in the throughout our convergence results.

\begin{definition}[Data-dependent aggregation parameter]
In Algorithm~\ref{alg:generalizedcocoa}, the aggregation parameter~$\gamma$ controls the level of adding ($\gamma:=1$) versus averaging ($\gamma:=\tfrac{1}{K}$) of the partial solutions from all machines.
For the convergence results discussed below to hold, the subproblem parameter $\sigma'$ must be chosen not smaller than
\begin{equation}
\label{eq:sigmaPrimeSafeDefinition} 
\sigma'
\geq
\sigma'_{min}
 \eqdef
 \aggpar
 \max_{\alphav\in \R^n}
 \frac{
 \|A\alphav\|^2}{\sum_{k=1}^K \|A\vsubset{\alphav}{k}\|^2} \, . \vspace{-1mm}
\end{equation}
\end{definition}

The simple choice of $\sigma' := \gamma K$ is valid for \eqref{eq:sigmaPrimeSafeDefinition}, i.e., 

\vspace{-1.5em}
\[
\gamma K \geq \sigma'_{min} \, .
\] In some cases, it will be possible to give better (data-dependent) choices for $\sigma'$, closer to the actual bound given in $\sigma'_{min}$.

Our first lemma in the overall proof of convergence helps to relate change in local subproblems to the global objective $\bD(\cdot)$.
\begin{lemma}
\label{lem:RelationOfDTOSubproblems}
For any dual
$\alphav, \Delta \alphav \in \R^n$, $\vv = \vv(\alphav) := A\alphav$, and real values $\aggpar,\sigma'$ satisfying~\eqref{eq:sigmaPrimeSafeDefinition}, it holds that
\begin{equation}
  \bD\Big(
\alphav +\aggpar 
\sum_{k=1}^K
\vsubset{\Delta \alphav}{k}\!
\Big) 
 \leq 
 (1-\aggpar) \bD(\alphav)  + \aggpar 
 \sum_{k=1}^K 
 \Ggk(\vsubset{\Delta \alphav}{k}; \vv, \vsubset{\alphav}{k}) \, .
\end{equation}
\end{lemma}
\begin{proof}
\allowdisplaybreaks
In this proof we follow the line of reasoning in \citep[Lemma 4]{Ma:2015ti} with a more general $(1/\tau)$ smoothness assumption on $f(\cdot)$. An outer iteration of \proxcocoa performs the following update:
\begin{align}
\label{eq:dualtosubproblem1}
\bD(\alphav+\gamma \sum_{k=1}^K\vsubset{\Delta \alphav}{k})
&= \underbrace{ f(\vv(\alphav + \gamma \sum_{k=1}^K \vsubset{\Delta \alphav}{k}))}_A +
\underbrace{\sum_{i=1}^{n}
\ell_i(\alpha_i +\gamma (\sum_{k=1}^K \vsubset{\Delta \alphav}{k})_i)}_{B} \, .
\end{align}

We bound the terms $A$ and $B$ separately. First we bound A using $(1/\tau)$-smoothness of $f$:
\begin{align*}
A &= f\Big(\vv(\alphav + \gamma \sum_{k=1}^K \vsubset{\Delta \alphav}{k})\Big) =
f\Big(\vv(\alphav) + \gamma \sum_{k=1}^K \vv(\vsubset{\Delta \alphav}{k})\Big) \\
& \overset{\text{smoothness of $f$ as in \eqref{eq:smooth}}}{\leq} f(\vv(\alphav)) +\sum_{k=1}^K \gamma \nabla f(\vv(\alphav))^T  \vv(\vsubset{\Delta \alphav}{k})  + \frac{\gamma^2}{2\tau} \|\sum_{k=1}^K \vv(\vsubset{\alphav}{k})\|^2 \\
& \overset{\text{definition of $\wv$ as in \eqref{eq:dualPdualrelation}}}{\leq} f(\vv(\alphav)) +\sum_{k=1}^K \gamma \vv(\vsubset{\Delta \alphav}{k})^T\wv(\alphav)  + \frac{\gamma^2}{2\tau} \|\sum_{k=1}^K \vv(\vsubset{\alphav}{k})\|^2 \\
& \overset{\text{safe choice of $\sigma'$ as in \eqref{eq:sigmaPrimeSafeDefinition}}}{\leq} f(\vv(\alphav)) +\sum_{k=1}^K \gamma \vv(\vsubset{\Delta \alphav}{k})^T\wv(\alphav)  + \frac{1}{2\tau}\gamma \sigma' \sum_{k=1}^K \| \vv(\vsubset{\alphav}{k})\|^2 \, .
\end{align*}  
 
Next we use Jensen's inequality to bound B:
\begin{align*}
B = \sum_{k=1}^K \left( \sum_{i\in \mathcal{P}_k} \ell_i(\alpha_i+\gamma   (\vsubset{\Delta \alphav}{k})_i) \right)
& = \sum_{k=1}^K \left( \sum_{i\in \mathcal{P}_k} \ell_i((1-\gamma)   \alpha_i+\gamma (\alphav + \vsubset{\Delta \alphav}{k})_i) \right) \\
&\leq  \sum_{k=1}^K \left( \sum_{i\in \mathcal{P}_k} (1-\gamma) \ell_i(\alpha_i) +\gamma \ell_i(\alphav_i + {\vsubset{\Delta \alphav}{k}}_i) \right) \, .
\end{align*}

Plugging $A$ and $B$ back into~\eqref{eq:dualtosubproblem1} yields:
\begin{align*}
\bD\Big(\alphav +\gamma \sum_{k=1}^K \vsubset{\Delta \alphav}{k}\Big) 
\le \ &  f(\vv(\alphav)) \pm \gamma f(\vv(\alphav))
+\sum_{k=1}^K \gamma  \vv(\vsubset{\Delta \alphav}{k})^T\wv(\alphav)   + \frac{1}{2\tau}\gamma \sigma' \sum_{k=1}^K \| \vv(\vsubset{\alphav}{k})\|^2 \\
& +  \sum_{k=1}^K\sum_{i\in \mathcal{P}_k} (1-\gamma) \ell_i(\alpha_i) +\gamma \ell_i(\alphav_i + {\vsubset{\Delta \alphav}{k}}_i) \\
= \ & \underbrace{(1-\gamma) f(\vv(\alphav)) +  \sum_{k=1}^K \left(\sum_{i\in \mathcal{P}_k} (1-\gamma) \ell_i(\alpha_i)  \right)}_{(1-\gamma) \bD(\alphav)} \\
&  +  \gamma \sum_{k=1}^K \left(\frac{1}{K} f(\vv(\alphav)) + \vv(\vsubset{\Delta \alphav}{k})^T\wv(\alphav) + \frac{\sigma'}{2\tau}  \| \vv( \vsubset{\alphav}{k})\|^2 + \sum_{i\in \mathcal{P}_k} \ell_i(\alphav_i + {\vsubset{\Delta \alphav}{k}}_i)  \right) \\
\overset{\eqref{eq:subproblem}}{=} \ & (1-\gamma) \bD(\alphav) +\gamma \sum_{k=1}^K \Ggk(  \vsubset{\Delta \alphav}{k}; \vv) \, , 
\end{align*}
where the last equality is by the definition of the subproblem objective $\Ggk(.)$ as in \eqref{eq:subproblem}.
\end{proof}

\subsection{Proof of Main Convergence Result (Theorem \ref{thm:convergenceNonsmooth})}

Before proving the main convergence results, we introduce several useful quantities, including the the following lemma, which characterizes the effect of iterations of Algorithm~\ref{alg:generalizedcocoa} on the duality gap for any chosen local solver of approximation quality $\Theta$.

\begin{lemma}
\label{lem:basic}
Let $\ell_i$ be strongly\footnote{Note that the case of weakly convex $\ell_i(.)$ is explicitly allowed here as well, as the Lemma holds for the case $\mu = 0$.} convex with convexity parameter $\mu \geq 0$ with respect to the norm $\|\cdot\|$, $\forall i\in[n]$.
Then for all iterations~$t$ of Algorithm~\ref{alg:generalizedcocoa} under Assumption~\ref{asm:theta}, and any $s\in [0,1]$, it holds that
\begin{align}
\label{eq:lemma:dualdecrease_vs_dualitygap}
&\E[ \bD(\vc{\alphav}{t}) - \bD(\vc{\alphav}{t+1})] \geq \gamma (1-\Theta) \Big(s G(\vc{\alphav}{t}) - \frac{\sigma's^2}{2\tau} \vc{R}{t} \Big) \, ,
\end{align}
where
\begin{align}
\label{eq:defOfR}
\vc{R}{t}&:= - \tfrac{ \tau \mu (1-s)}{\sigma' s } \|\vc{\uv}{t}-\vc{\alphav}{t}\|^2 + \textstyle{\sum}_{k=1}^K \| A \vsubset{  (\vc{\uv}{t}  - \vc{\alphav}{t} )}{k}\|^2 \, ,
\end{align}
for $\vc{\uv}{t} \in\R^n$ with
\begin{equation}
\label{eq:defintionOfUi}
\vc{u_i}{t} 
\in \partial \ell^*_i(-\xv_i^T\wv(\vc{\alphav}{t})) \, .
\end{equation}
\end{lemma}
\begin{proof}
The line of proof is motivated by \citep[Lemma 19]{ShalevShwartz:2013wl} and follows \citep[Lemma 5]{Ma:2015ti}, with a main addition being the extension to our generalized subproblems $\Ggk(\cdot;\vv, \vsubset{\alphav}{k})$ along with the general mappings $\wv(\alphav) := \nabla f(\vv(\alphav))$ with $\vv(\alphav) := A \alphav$. 

For simplicity, we write $\alphav$ instead of $\vc{\alphav}{t}$, $\vv$ instead of $\vv(\vc{\alphav}{t})$, $\wv$ instead of $\wv(\vc{\alphav}{t})$ and $\uv$ instead of $\vc{\uv}{t}$. We can estimate the expected change of the objective $\bD(\alphav)$ as follows. Starting from the definition of the update $\vc{\alphav}{t+1} := \vc{\alphav}{t} + \gamma \, \sum_k \vsubset{\Delta \alphav}{k}$ from Algorithm~\ref{alg:generalizedcocoa}, we apply Lemma \ref{lem:RelationOfDTOSubproblems}, which relates the local approximation $\Ggk(\alphav;\vv, \vsubset{\alphav}{k})$ to the global objective $\bD(\alphav)$, and then bound this using the notion of quality of the local solver ($\Theta$), as in Assumption \ref{asm:theta}. This gives us: 
\begin{align*}
\E \big[\bD(\vc{\alphav}{t}) - \bD(\vc{\alphav}{t+1})\big] &= \E \Big[\bD(\alphav) - \bD\Big(\alphav + \gamma \sum_{k=1}^K \vsubset{\Delta \alphav}{k}\Big)\Big] \\
&\ge \gamma (1-\Theta) \left( \underbrace{ \bD(\alphav) - \sum_{k=1}^K \Ggk(\vsubset{\Delta \alphav^{\star}}{k}; \vv, \vsubset{\alphav}{k}) }_{C} \right) \, .
\tagthis
\label{eq:basic1}
\end{align*} 

We next upper bound the $C$ term, denoting $\Delta \alphav^{\star} = \sum_{k=1}^K \vsubset{\Delta \alphav^{\star}}{k}$. We first plug in the definition of the objective $\bD$ in \eqref{eq:obj} and the local subproblems \eqref{eq:subproblem}, and then substitute $s(u_i-\alpha_i)$ for $\Delta \alphav^{\star}_i$ and apply the $\mu$-strong convexity of the $\ell_i$ terms. This gives us:
\begin{align*}
C& = \sum_{i =1}^n \left(\ell_i(\alpha_i) - \ell_i(\alpha_i + \Delta \alphav^{\star}_i) \right) - (A\Delta \alphav^{\star})^T\wv(\alphav)  - \sum_{k=1}^K  \frac{\sigma'}{2\tau} \Big\|A\vsubset{\Delta \alphav^{\star}}{k}\Big\|^2 \\
&\ge \sum_{i =1}^n \left( s \ell_i(\alpha_i ) -s \ell_i(u_i ) + \frac{\mu}{2} (1-s)s (u_i -\alpha_i)^2 \right) \\
&\qquad - A(s (\uv  - \alphav ))^T\wv(\alphav) - \sum_{k=1}^K \frac{\sigma'}{2\tau}   \Big\|A(\vsubset{s (\uv  - \alphav )}{k})\Big\|^2 \, .\tagthis
\end{align*}

From the definition of the primal and dual optimization problems \eqref{eq:obj} and \eqref{eq:dualP}, and definition of convex conjugates, we can write the duality gap as:
\begin{align*}
\gap(\alphav) := \bP(\wv(\alphav))-(-\bD(\alphav))
&\overset{\eqref{eq:obj},\eqref{eq:dualP}}{=} \sum_{i=1}^n \left( \ell^*_i(- \xv_i^T\wv(\alphav)) + \ell_i(\alpha_i) \right) + f^*(\wv(\alphav)) + f(A\alphav)) \\
& = \sum_{i=1}^n \left( \ell^*_i( -\xv_i^T\wv(\alphav)) + \ell_i(\alpha_i) \right) + f^*(\nabla f(A\alphav)) + f(A\alphav) \\
& = \sum_{i=1}^n \left( \ell^*_i( -\xv_i^T\wv(\alphav)) + \ell_i(\alpha_i) \right) +  (A\alphav)^T\wv(\alphav)  \\
& = \sum_{i=1}^n \left( \ell^*_i( -\xv_i^T\wv(\alphav)) +  \ell_i(\alpha_i) + \alpha_i \xv_i^T\wv(\alphav) \right) \, .
\tagthis
\label{eq:basic3}
\end{align*}

The convex conjugate maximal property from \eqref{eq:defintionOfUi} implies that
\begin{equation}
\label{eq:basic2}
\ell_i(u_i) = u_i (-\xv_i^T\wv(\alphav)) -\ell^*_i(-\xv_i^T \wv(\alphav)) \, .
\end{equation}

Using \eqref{eq:basic2} and \eqref{eq:basic3}, we therefore have:
\begin{align*}
C &\overset{ \eqref{eq:basic2}}{\geq} \sum_{i =1}^n \left(s  \ell_i(\alpha_i ) - s u_i (-\xv_i^T\wv(\alphav)) + s\ell^*_i(-\xv_i^T\wv(\alphav)) + 
\frac{\mu}{2} (1-s)s (u_i -\alpha_i)^2 \right) \\
& \qquad  - A(s (\uv  - \alphav ))^T\wv(\alphav) - \sum_{k=1}^K \frac{\sigma'}{2\tau}   \Big\|A(\vsubset{s (\uv  - \alphav )}{k})\Big\|^2 \\
&= \sum_{i =1}^n  \big[  s\ell_i(\alpha_i ) + s\ell^*_i(-\xv_i^T \wv(\alphav)) + s \xv_i^T \wv(\alphav) \alpha_i \big]  - \sum_{i =1}^n \big[  s \xv_i^T \wv(\alphav) ( \alpha_i-u_i ) - \frac{\mu}{2}(1-s)s (u_i -\alpha_i)^2 \big] \\
&\qquad  - A(s (\uv  - \alphav ))^T\wv(\alphav) - \sum_{k=1}^K \frac{\sigma'}{2\tau}   \Big\|A(\vsubset{s (\uv  - \alphav )}{k})\Big\|^2 \\
&\overset{\eqref{eq:basic3}}{=} s \gap(\alphav) + \frac{\mu}{2} (1-s)s  \|\uv-\alphav\|^2  - \frac{\sigma's^2}{2\tau} \sum_{k=1}^K   \| A \vsubset{  (\uv  - \alphav )}{k}\|^2 \, .
\tagthis 
\label{eq:basic4}
\end{align*}

The claimed improvement bound \eqref{eq:lemma:dualdecrease_vs_dualitygap} then follows by plugging \eqref{eq:basic4} into \eqref{eq:basic1}.
\end{proof}

The following Lemma provides a uniform bound on~$\vc{R}{t}$:

\begin{lemma}
\label{lemma:BoundOnR}
If $\ell^*_i$ are $L$-Lipschitz continuous for all $i\in [n]$, then
\begin{equation}
\label{eq:asfjoewjofa}
\forall t:  \vc{R}{t} \leq 4L^2 \underbrace{\sum _{k=1}^K \sigma_k  n_k}_{=: \sigma}\, , \end{equation}
where
\begin{equation}
\label{eq:definitionOfSigmaK}
\sigma_k \eqdef \max_{\vsubset{\alphav}{k} \in \R^n} \frac{\|A \vsubset{\alphav}{k}\|^2}{\|\vsubset{\alphav}{k}\|^2} \, .
\end{equation}
\end{lemma}
\begin{proof}
\citep[Lemma 6]{Ma:2015ti}. 
For general convex functions, the strong convexity parameter is 
$\mu=0$, and hence the definition \eqref{eq:defOfR} of the complexity constant $\vc{R}{t}$ becomes
\begin{align*} 
\vc{R}{t}
=
  \sum _{k=1}^K   
  \| A \vsubset{  (\vc{\uv} {t} - \vc{\alphav}{t} )}{k}\|^2
\overset{\eqref{eq:definitionOfSigmaK}}{\leq}   
\sum _{k=1}^K 
\sigma_k  
  \|   \vsubset{  (\vc{\uv} {t} - \vc{\alphav}{t} )}{k}\|^2
\leq
\sum _{k=1}^K 
\sigma_k  |\mathcal{P}_k| 4L^2 \, .
\end{align*}
Here the last inequality follows from  in \citep[Lemma 21]{ShalevShwartz:2013wl}, which shows that for $\ell^*_i : \R \to \R$ being $L$-Lipschitz, it holds that for any real value $a$ with $|a|> L$ one has that
$\ell_i(a) = +\infty$.
\end{proof}

\begin{remark}
\label{rmk:asfwaefwae}
\citep[Remark 7]{Ma:2015ti} If all data points $\xv_i$ are normalized such that $\|\xv_i\|\leq 1$ $\forall i\in [n]$, then $\sigma_k \leq |\mathcal{P}_k| = n_k$. Furthermore, if we assume that the data partition is balanced, i.e., that $n_k = n/K$ for all $k$, then $\sigma \le n^2/K$. This can be used to bound the constants $\vc{R}{t}$, above, as $ \vc{R}{t} \leq  \frac{4L^2 n^2}{K}.$
\end{remark}

\begin{theorem}
\label{thm:convergenceNonsmoothCoCoA}

Consider Algorithm \ref{alg:generalizedcocoa}, using a local solver of quality $\Theta$ (See Assumption \ref{asm:theta}).
Let $\ell^*_i(\cdot)$ be $L$-Lipschitz continuous,
and $\epsilon_G>0$ be the desired duality gap (and hence an upper-bound on  suboptimality $\epsilon_{\bD}$).
Then after $T$ iterations, where
\begin{align}\label{eq:dualityRequirements}
T
&\geq
T_0 +
\max\{\Big\lceil \frac1{\gamma (1-\Theta)}\Big\rceil,\frac
{4L^2  \sigma   \sigma'}
{ \tau \epsilon_G
\gamma (1-\Theta)}\} \, ,
\\
T_0
\geq t_0+
\Big[
\frac{2}{ \gamma (1-\Theta) }
&\left(
\frac
{8L^2  \sigma   \sigma'}
{\tau \epsilon_G}
-1
\right)
\Big]_+\, ,\notag \, \, \, \, t_0  \geq
  \max(0,\Big\lceil \tfrac1{\gamma (1-\Theta)}
\log\left(
\tfrac{\tau(
 \bD(\vc{\alphav}{0})-\bD(\alphav^{\star} ))
  }{2 L^2 \sigma \sigma'}
  \right)
 \Big\rceil)\, ,\notag
\end{align}
we have that the expected duality gap satisfies
\[
\Exp[\bP( \wv(\overline\alphav)) - (-\bD(\overline \alphav)) ] \leq \epsilon_G
\]
at the averaged iterate
\begin{equation}\label{eq:averageOfAlphaDefinition}
\overline \alphav: = \tfrac1{T-T_0}\textstyle{\sum}_{t=T_0+1}^{T-1} \vc{\alphav}{t} \, .
\end{equation}
\end{theorem}

\begin{proof} This proof draws from the line of reasoning in \citep[Theorem 2]{ShalevShwartz:2013wl} and follows \citep[Theorem 8]{Ma:2015ti} but for the more general problem setting \eqref{eq:obj}. 
We begin by estimating the expected change of feasibility for $\bD$. We can bound this above by using Lemma \ref{lem:basic} and the fact that the $\bP(\cdot)$ is always a lower bound for $-\bD(\cdot)$, and then applying \eqref{eq:asfjoewjofa} to find:
\begin{align*} 
 \Exp[\bD(\vc{\alphav}{t+1})-\bD(\alphav^{\star})]
 &
\leq
\left( 
 1-\aggpar
(1-\Theta)
   s
\right) 
   (\bD(\vc{\alphav}{t})-\bD(\alphav^{\star}))
+
\aggpar
(1-\Theta) 
 \tfrac{\sigma' s^2}{2\tau}
4L^2  \sigma \, .
\tagthis 
\label{eq:asoifejwofa}
\end{align*}

 Using
\eqref{eq:asoifejwofa}
recursively we have 
 \begin{align*} 
 \Exp[\bD(\vc{\alphav}{t})-\bD(\alphav^{\star})]
&\leq
\left( 
 1-\aggpar
(1-\Theta)
   s
\right)^t 
   (\bD(\vc{\alphav}{0})-\bD(\alphav^{\star} ))
+
 s
\frac{4L^2  \sigma   \sigma'}{2\tau} \, . 
\tagthis
\label{eq:asfwefcaw}  
 \end{align*}
Choosing 
$s=1$ and $t= t_0:= \max\{0,\lceil  
\frac1{\aggpar (1-\Theta)}
\log(
 2 (\bD(\vc{\alphav}{0})-\bD(\alphav^{\star} ))
  / (4 L^2 \sigma \sigma')
  )
 \rceil\}$
will lead to 
\begin{align}\label{eq:induction_step1}
  \Exp[\bD(\vc{\alphav}{t})-\bD(\alphav^{\star})]
 &\leq  
\left( 
 1-\aggpar
(1-\Theta)  
\right)^{t_0}
  (\bD(\vc{\alphav}{0})-\bD(\alphav^{\star} ))
+ 
\frac{4L^2  \sigma   \sigma'}{2\tau}
\le \frac{4L^2  \sigma   \sigma'}{\tau} \, .
\end{align} 
Next, we show inductively that 
\begin{align}
\label{eq:expectationOfDualFeasibility}
\forall t\geq t_0 :  \Exp[\bD(\vc{\alphav}{t})-\bD(\alphav^{\star} )]
&\leq 
\frac{4L^2  \sigma   \sigma'}{\tau( 1+ \frac12  \aggpar (1-\Theta)  (t-t_0))} \, .
\end{align}
Clearly, \eqref{eq:induction_step1} implies that \eqref{eq:expectationOfDualFeasibility} holds for $t=t_0$.
Assuming that it holds for any $t\geq t_0$, we show that it must also hold for $t+1$. 
Indeed, using 
\begin{equation}
\label{eq:asdfjoawjdfas}
s=
\frac{1}
 {1+ \frac12 \aggpar (1-\Theta) (t-t_0)} \in [0,1] \, ,
\end{equation} 
  we obtain
\begin{align*}
\Exp[
\bD(\vc{\alphav}{t+1})-\bD(\alphav^{\star} )]
\le \frac{4L^2  \sigma   \sigma'}{\tau}
\underbrace{\left( 
\frac{
1+ \frac12 \aggpar (1-\Theta) (t-t_0)
-\frac12 \aggpar
(1-\Theta)
}
 {(1+ \frac12 \aggpar (1-\Theta) (t-t_0))^2}
\right)}_{D} \, 
\end{align*}
by applying the bounds \eqref{eq:asoifejwofa} and \eqref{eq:expectationOfDualFeasibility}, plugging in the definition of $s$ \eqref{eq:asdfjoawjdfas}, and simplifying. We upperbound the term $D$ using the fact that geometric mean
 is less or equal to arithmetic mean:
\begin{align*}
D&=
\frac1
{1+ \frac12 \aggpar (1-\Theta) (t+1-t_0)}
\underbrace{ 
\frac{
(1+ \frac12 \aggpar (1-\Theta) (t+1-t_0))
(1+ \frac12 \aggpar (1-\Theta) (t-1-t_0))
}
 {(1+ \frac12 \aggpar (1-\Theta) (t-t_0))^2}}_{\leq 1}
 \\
&\leq  
\frac1
{1+ \frac12 \aggpar (1-\Theta) (t+1-t_0)} \, ,
\end{align*}
 
If $\overline \alphav$ is defined as \eqref{eq:averageOfAlphaDefinition}, we apply the results of Lemma~\ref{lem:basic} and Lemma~\ref{lemma:BoundOnR} to obtain
\begin{align*}
\Exp[\gap(\overline\alphav)] &=  
 \Exp\left[\gap\left(\sum_{t=T_0}^{T-1} \tfrac1{T-T_0} \vc{\alphav}{t}\right)\right]
 \leq
  \tfrac1{T-T_0} \Exp\left[\sum_{t=T_0}^{T-1} \gap\left( \vc{\alphav}{t}\right)\right]
\\
&\leq
\frac1{\aggpar
(1-\Theta)
 s}
   \frac1{T-T_0} 
   \Exp\left[
\bD(\vc{\alphav}{T_0})
-
\bD(\alphav^{\star})
  \right] 
+\tfrac{4L^2 \sigma \sigma' s}{2\tau} \, .  
\tagthis \label{eq:askjfdsanlfas}
  \end{align*}
If $T\geq \lceil
\frac1{\aggpar (1-\Theta)}\rceil+T_0$ such that $T_0\geq t_0$
we have
\begin{align*}
\Exp[\gap(\overline\alphav)] 
&\overset{\eqref{eq:askjfdsanlfas}
,\eqref{eq:expectationOfDualFeasibility}
}{\leq}
\frac1{\aggpar
(1-\Theta)
 s}
   \frac1{T-T_0} 
\left(
\frac{4L^2  \sigma   \sigma'}{\tau( 1+ \frac12  \aggpar (1-\Theta)  (T_0-t_0))}
\right)
+\frac{4L^2 \sigma \sigma' s}{2\tau}
\\
&=
\frac{
4L^2  \sigma   \sigma'}{\tau}
\left(
\frac1{\aggpar
(1-\Theta)
 s}
   \frac1{T-T_0} 
\frac{1}{ 1+ \frac12  \aggpar (1-\Theta)  (T_0-t_0)}
+\frac{  s}{2 }
\right) \, . 
\tagthis
\label{eq:fawefwafewa}
\end{align*}
Choosing 
\begin{equation}
\label{eq:afskoijewofaw}
s=\frac{1}{(T-T_0) \aggpar (1-\Theta)} \in [0,1]
\end{equation}
gives us
\begin{align*}
\Exp[\gap(\overline\alphav)] 
&
\overset{\eqref{eq:fawefwafewa},
\eqref{eq:afskoijewofaw}}{\leq}
\frac{4L^2  \sigma   \sigma'}{\tau}
\left(
\frac{1}{ 1+ \frac12  \aggpar (1-\Theta)  (T_0-t_0)}
+\frac{1}{(T-T_0) \aggpar (1-\Theta)} \frac{  1}{2 }
\right) \, . \tagthis
\label{eq:afsjweofjwafea}
\end{align*}
To have right hand side of
\eqref{eq:afsjweofjwafea}
smaller then 
$\epsilon_\gap$
it is sufficient to choose
$T_0$ and $T$ such that
\begin{eqnarray}
\label{eq:sfadwafeewafa}
\frac{4L^2  \sigma   \sigma'}{\tau}
\left(
\frac{1}{ 1+ \frac12  \aggpar (1-\Theta)  (T_0-t_0)}
\right)
&\leq & \frac12 \epsilon_\gap \, ,
\\
\label{eq:sfadwafeewafa2}
\frac{4L^2  \sigma   \sigma'}{\tau}
\left(
\frac{1}{(T-T_0) \aggpar (1-\Theta)} \frac{  1}{2 }
\right)
&\leq & \frac12 \epsilon_\gap \, .
\end{eqnarray}
Hence if
\begin{eqnarray*}
t_0+
\frac{2}{ \aggpar (1-\Theta) }
\left(
\frac
{8L^2  \sigma   \sigma'}
{\tau\epsilon_\gap}
-1
\right)
&\leq & 
 T_0 
\, , \text{ and}
\\
T_0
+
\frac
{4L^2  \sigma   \sigma'}
{\tau\epsilon_\gap
\aggpar (1-\Theta)}
&\leq &  T \, ,
\end{eqnarray*}
then 
\eqref{eq:sfadwafeewafa}
and
\eqref{eq:sfadwafeewafa2}
are satisfied.
\end{proof}

The following main theorem simplifies the results of Theorem~\ref{thm:convergenceNonsmoothCoCoA} and is a generalization of \citep[Corollary 9]{Ma:2015ti} for general $f^*(\cdot)$ functions:

\begin{reptheorem}{thm:convergenceNonsmooth}
Consider Algorithm \ref{alg:generalizedcocoa} with $\gamma :=1$, using a local solver of quality $\Theta$ (see Assumption \ref{asm:theta}). Let $\ell^*_i(\cdot)$ be $L$-Lipschitz continuous, and assume that the columns of $A$ satisfy $\|\xv_i\|\leq 1$ $\forall i\in [n]$.
Let $\epsilon_{G}>0$ be the desired duality gap (and hence an upper-bound on primal sub-optimality).
Then after $T$ iterations, where
\begin{align}\label{eq:dualityRequirements}
T
&\geq
T_0 +
\max\{\Big\lceil \frac1{1-\Theta}\Big\rceil,
\frac{4L^2n^2}{\tau\epsilon_{G}(1-\Theta)}\} \,,
\\
T_0
&\geq t_0+
\Big[
\frac{2}{ 1-\Theta }
\left(\frac {8L^2n^2} {\tau\epsilon_{ G}}
-1
\right)
\Big]_+ \, ,\notag
\\
t_0 &\geq
  \max(0,\Big\lceil \tfrac1{(1-\Theta)}
\log\left(
\tfrac{
 \tau({\bD}(\vc{\alphav}{0})-{\bD}(\alphav^{\star} ))
  }{2 L^2 Kn}
 \right)
 \Big\rceil)\,,\notag
\end{align}
we have that the expected duality gap satisfies
\[
\Exp[\bP( \wv(\overline\alphav)) - (-\bD(\overline \alphav)) ] \leq \epsilon_{ G}
\]
(where $\overline\alphav$ is the averaged iterate returned by Algorithm \ref{alg:generalizedcocoa}).
\end{reptheorem}

\begin{proof} Plug in parameters $\gamma := 1$, $\sigma' := \gamma K = K$ to the results of Theorem \ref{thm:convergenceNonsmoothCoCoA}, and note that for balanced datasets we have $\sigma \le \frac{n^2}{K}$ (see Remark \ref{rmk:asfwaefwae}). We can further simplify the rate by noting that $\tau = 1$ for the 1-smooth losses (least squares and logistic) given as examples in this work.
\end{proof}

\begin{remark} For pure $L_1$-regularized problems as discussed in Section \ref{sec:l1}, we have that the above theorem directly delivers a primal-dual convergence with a sublinear rate. This is because in view of Lemma~\ref{lem:l1surrogate}, we know that $\ell^*_i$ is $B$-Lipschitz for the bounded support modification introduced in Section~\ref{sec:convergence}.
\end{remark}

\subsection{Proof of Convergence Result for Strongly Convex $\ell_i$}
Our second main theorem follows reasoning in \cite{ShalevShwartz:2013wl} and is a generalization of \citep[Corollary 11]{Ma:2015ti}. We first introduce a lemma to simplify the proof.

\begin{lemma}
\label{lemma:asfewfawfcda}
Assume that $\ell_i(0) \in [0,1]$ for all $i\in[n]$, then for the zero vector $\vc{\alphav}{0}
 := {\bf 0}\in \R^n$, we have
\begin{equation}
\label{eq:afjfjaoefvcwa}
\bD(\vc{\alphav}{0})-\bD(\alphav^{\star})
= 
\bD({\bf 0})-\bD(\alphav^{\star})
 \leq n \, .
 \end{equation}
\end{lemma}
\begin{proof}
For $\alphav := {\bf 0}\in \R^n$, we have
$\wv(\alphav) = A \alphav 
 = {\bf 0} \in \R^d$.
 Therefore, since the dual $-\bD(\cdot)$ is always a lower bound on the primal $\bP(\cdot)$, and by definition of the objective $\bD$ given in~\eqref{eq:obj},
\begin{align*}
0 &\leq \bD(\alphav)-\bD(\alphav^{\star})
\leq \bP(\wv(\alphav)) - (-\bD(\alphav))
 \overset{\eqref{eq:obj}
}{\leq} n \, . \qedhere
\end{align*} 
\end{proof}

\begin{theorem}
\label{thm:convergenceSmoothCase}
Assume that $\ell_i$ are $\mu$-strongly convex $\forall i\in[n]$.
We define $\sigma_{\max} = 
\max_{k\in[K]} \sigma_k$. Then after $T$ iterations of Algorithm \ref{alg:generalizedcocoa}, with
\[
 T
    \geq 
\tfrac{1}
   {\aggpar
(1-\Theta)}
\tfrac
{\mu\tau+
\sigma_{\max} \sigma'}
{\mu\tau }
    \log \tfrac n {\epsilon_\bD} \, , 
\]
it holds that
\[\Exp[\bD(\vc{\alphav}{T})-\bD(\alphav^{\star})]
   \leq \epsilon_\bD \, .\]
Furthermore, after $T$ iterations with
\[
 T 
    \geq 
\tfrac{1}
   {\aggpar
(1-\Theta)}
\tfrac
{\mu\tau+
\sigma_{\max} \sigma'}
{ \mu \tau}
    \log 
\left(
\tfrac{1}
   {\aggpar
(1-\Theta)}
\tfrac
{\mu\tau+
\sigma_{\max} \sigma'}
{ \mu \tau}
    \tfrac n {\epsilon_\gap}
    \right)\, ,
\]
we have the expected duality gap
\[
\Exp[
\bP( \wv(\vc{\alphav}{T})) - (-\bD(\vc{\alphav}{T}))
]\leq \epsilon_\gap \, .
\]
\end{theorem}

\begin{proof}
Given that $\ell_i(.)$ is $\mu$-strongly convex with respect to the
$\|\cdot\|$ norm, we can apply \eqref{eq:defOfR} and the definition of $\sigma_k$ to find:
\begin{align*}
\vc{R}{t}& 
\leq
-
\tfrac{ \tau \mu  (1-s)}{\sigma' s }
   \|\vc{\uv}{t}-\vc{\alphav}{t}\|^2 
+ 
 {\sum}_{k=1}^K   
 \sigma_k
  \|  \vsubset{   \vc{\uv}{t}  - \vc{\alphav}{t}  }{k}\|^2
\\
&\leq
\left(
-
\tfrac{ \tau \mu (1-s)}{\sigma' s }
+\sigma_{\max}
\right)
   \|\vc{\uv}{t}-\vc{\alphav}{t}\|^2 \, ,\tagthis
   \label{eq:afjfocjwfcea} 
\end{align*}
where $\sigma_{\max} = \max_{k\in[K]} \sigma_k$. If we plug the following value of $s$
 \begin{equation}
 s=
  \frac{ \tau \mu }
      {\tau \mu +
\sigma_{\max} \sigma'}\in [0,1]
\label{eq:fajoejfojew}
\end{equation} 
into
\eqref{eq:afjfocjwfcea}
we obtain that
$\forall t: \vc{R}{t}\leq 0$.
Putting the  same $s$
into
\eqref{eq:lemma:dualdecrease_vs_dualitygap}
will give us
\begin{align*}
&\Exp[
\bD(\vc{\alphav}{t})
-
\bD(\vc{\alphav}{t+1})
 ]
\overset{\eqref{eq:lemma:dualdecrease_vs_dualitygap}
,\eqref{eq:fajoejfojew}}{\geq}
\aggpar
(1-\Theta)
 \frac{ \tau \mu }
      {\tau \mu +
\sigma_{\max} \sigma'} \gap(\vc{\alphav}{t})
\geq
\aggpar
(1-\Theta)
 \frac{\tau \mu  }
      {\tau \mu +
\sigma_{\max} \sigma'} (\bD(\vc{\alphav}{t})-\bD(\alphav^{\star})) \, .
\tagthis
\label{eq:fasfawfwaf}
\end{align*}
Using the fact that
$\Exp[\bD(\vc{\alphav}{t})-\bD(\vc{\alphav}{t+1})]
=\Exp[\bD(\alphav^{\star})-\bD(\vc{\alphav}{t+1})]
+\bD(\vc{\alphav}{t})-\bD(\alphav^{\star})
$
we have 
\begin{align*}
\Exp[\bD(\alphav^{\star})-\bD(\vc{\alphav}{t+1})]
+\bD(\vc{\alphav}{t})-\bD(\alphav^{\star})
\overset{
\eqref{eq:fasfawfwaf}}
{
\geq
}
\aggpar
(1-\Theta)
 \frac{ \tau \mu  }
      {\tau \mu+
\sigma_{\max} \sigma'}(\bD(\vc{\alphav}{t})- \bD(\alphav^{\star})) \, ,
\end{align*}
which is equivalent to
\begin{align*}
\Exp[\bD(\vc{\alphav}{t+1})-\bD(\alphav^{\star})]
\leq 
\left(
1-\aggpar
(1-\Theta)
 \frac{\tau \mu  }
      {\tau \mu +
\sigma_{\max} \sigma'}\right)
(\bD(\vc{\alphav}{t})-\bD(\alphav^{\star})) \, .
\tagthis \label{eq:affpja}
\end{align*}
Therefore if we denote $\vc{\epsilon_\bD}{t} = \bD(\vc{\alphav}{t})-\bD(\alphav^{\star})$
we have recursively that
\begin{align*}
 \Exp[\vc{\epsilon_\bD}{t}] 
 & \overset{\eqref{eq:affpja}}{\leq}   \left(
 1-\aggpar
(1-\Theta)
 \frac{ \tau \mu }
      {\tau \mu +
\sigma_{\max} \sigma'}
   \right)^t \vc{\epsilon_\bD}{0}
\overset{\eqref{eq:afjfjaoefvcwa}}{\leq}
\left(
 1-\aggpar
(1-\Theta)
 \frac{ \tau \mu }
      {\tau \mu +
\sigma_{\max} \sigma'}
   \right)^t n \\
& \leq \exp\left(-t \aggpar
(1-\Theta)
 \frac{ \tau \mu }
      {\tau \mu +
\sigma_{\max} \sigma'}
     \right)n \, .
\end{align*}
The right hand side will be smaller than some $\epsilon_\bD$ if 
\[
 t   
    \geq 
\frac{1}
   {\aggpar
(1-\Theta)}
\frac
{\tau \mu +
\sigma_{\max} \sigma'}
{ \tau \mu  }
    \log \frac n {\epsilon_\bD} \, .
\]
Moreover, to bound the duality gap, we have
\begin{align*}
\aggpar
(1-\Theta)
 \frac{ \tau \mu }
      {\tau \mu +
\sigma_{\max} \sigma'} \gap(\vc{\alphav}{t})
&
\overset{
\eqref{eq:fasfawfwaf}
}{\leq}
\Exp[
\bD(\vc{\alphav}{t})
-
\bD(\vc{\alphav}{t+1})
 ]
\leq 
\Exp[
\bD(\vc{\alphav}{t})-\bD(\alphav^{\star})
 ] \, .  
\end{align*}
Thus,  $\gap(\vc{\alphav}{t})\leq 
\frac1{
\aggpar
(1-\Theta)}
 \frac      {\tau \mu +
\sigma_{\max} \sigma'} 
{ \tau\mu  }    \vc{\epsilon_\bD}{t}$.  
Hence if $\epsilon_\bD \leq 
\aggpar
(1-\Theta)
 \frac{\tau \mu }
      {\tau\mu +
\sigma_{\max} \sigma'} 
 \epsilon_\gap $
then $\gap(\vc{\alphav}{t})\leq \epsilon_\gap$.
Therefore
after 
\[
 t   
    \geq 
\frac{1}
   {\aggpar
(1-\Theta)}
\frac
{\tau \mu+
\sigma_{\max} \sigma'}
{ \tau \mu  }
    \log 
\left(
\frac{1}
   {\aggpar
(1-\Theta)}
\frac
{\tau \mu+
\sigma_{\max} \sigma'}
{ \tau \mu  }
    \frac n {\epsilon_\gap}
    \right) 
\]
iterations we have obtained a duality gap less than $\epsilon_\gap$.
\end{proof}

\begin{reptheorem}
{thm:convergenceSmooth}
Consider Algorithm~\ref{alg:generalizedcocoa} with $\gamma := 1$, using a local solver of quality $\Theta$ (See Assumption \ref{asm:theta}). Let $\ell_i(\cdot)$ be  $\mu$-strongly convex $\forall i\in[n]$, and assume that the columns of $A$ satisfy $\|\xv_i\|\leq 1$ $\forall i\in [n]$. Then we have that $T$ iterations are sufficient for suboptimality 
$\epsilon_\bD$, with
\[
T \geq
\tfrac{1}
   {
\gamma(1-\Theta)}
\tfrac
{\tau \mu+n}
{\tau \mu }
    \log \tfrac n {\epsilon_\bD} \, . 
\]
Furthermore, after $T$ iterations with
\[
 T 
    \geq 
\tfrac{1}
   {\aggpar
(1-\Theta)}
\tfrac
{\tau \mu+
n}
{ \tau \mu }
    \log 
\left(
\tfrac{1}
   {\aggpar
(1-\Theta)}
\tfrac
{\tau \mu+
n}
{ \tau \mu }
    \tfrac n {\epsilon_\gap}
    \right)\, ,
\]
we have the expected duality gap
\[
\Exp[
\bP( \wv(\vc{\alphav}{T})) - \bD(\vc{\alphav}{T})
]\leq \epsilon_\gap \, .
\]
\end{reptheorem}

\begin{proof}
Plug in parameters $\gamma := 1$, $\sigma' := \gamma K = K$ to the results of Theorem \ref{thm:convergenceSmoothCase} and note that for balanced datasets we have $\sigma_{\max} \le \frac{n}{K}$ (see Remark \ref{rmk:asfwaefwae}). We can further simplify the rate by noting that $\tau = 1$ for the 1-smooth losses (least squares and logistic) given as examples in this work.
\end{proof} 

\begin{remark} For elastic net regularized problems as discussed in Section \ref{sec:elasticnet}, we have that the above theorem directly delivers a primal-dual convergence with a geometric rate. This is because in view of Lemma~\ref{lem:elasticnetconjugate}, we know that $\ell^*_i$ is $1/\eta$-smooth for any elastic net parameter $\eta \in (0,1]$.
\end{remark}

\vspace{2mm}
\section{Recovering \cocoap as a Special Case}
\label{sec:oldcocoa}

As a special case, \proxcocoa directly applies to any $L_2$-regularized loss-minimization problem, including those presented in \cite{Jaggi:2014vi,Ma:2015ti}. In this setting, the original machine-learning problem is mapped to what we here refer to as the ``dual'' problem formulation \eqref{eq:dualP}:\vspace{-1mm}
\[
    \min_{\wv \in \R^{d}} \quad \Big[ \
    \bP(\wv) := f^*(\wv )
    + \sum_{i=1}^{n} \ell^*_i(-\xv_i^T\wv) \ \Big] \, ,\vspace{-1mm}
\]
with $f^*(\cdot)=\tfrac\lambda2 \|\cdot\|^2$ being the regularizer, and~$\ell^*_i$ taking the role of loss function, acting on a linear predictor $\xv_i^T\wv$ (recall that $\xv_i$ is a column of the data matrix $A$).
In other words, the \proxcocoa algorithm will in this case apply to \eqref{eq:obj} as the dual of the original input problem (which will be mapped to \eqref{eq:dualP}), as described in \cite{Jaggi:2014vi,Ma:2015ti}. The following remarks show that we recover the linear (geometric) convergence rates for smooth loss functions $\ell^*_i$, and sublinear convergence for Lipschitz losses. 
Note that this contrasts the discussed applications of \proxcocoa where the $g$ function has the role of the regularizer instead.
\vspace{1mm}

\begin{remark}
If we view~\eqref{eq:dualP} as the primal
, restrict $f^*(\cdot) := \tfrac{\lambda}{2}\| \cdot \|^2$ (so that $\tau = \lambda$), and let $\ell^*_i := \tfrac{1}{n} \oldell^*_i$,  Theorem~\ref{thm:convergenceNonsmooth} recovers as a special case the \cocoap rates for general $L$-Lipschitz $\oldell^*_i$ losses (see \citep[Corollary 9]{Ma:2015ti}). \end{remark}

This follows since $\ell^*_i$ is $L$-Lipschitz if and only if 
$\ell_i$ has $L$-bounded support \citepsup[Corollary 13.3.3]{Rockafellar:1997ww}.

\begin{remark}
If we view~\eqref{eq:dualP} as the primal
, restrict $f^*(\cdot) := \tfrac{\lambda}{2}\| \cdot \|^2$ (so that $\tau = \lambda$), and scale $\ell^*_i := \tfrac{1}{n} \oldell^*_i$,  Theorem~\ref{thm:convergenceSmooth} recovers as a special case the \cocoap rates for $(1/\oldell^*_i)$-smooth losses (see \citep[Corollary 11]{Ma:2015ti}). \end{remark}

This follows since $\ell^*_i$ is $\mu$-strongly convex if and only if~$\ell_i$ is $(1/{\mu})$-smooth \citepsup[Theorem 6]{Kakade:2009wh}.

\begin{remark}Note that the approach of mapping the original objective to \eqref{eq:dualP} does \emph{not} allow general regularizers such as $\Lone$. This is one of the reasons we have proposed swapping the roles of regularizers and losses, and running \proxcocoa on the primal of the original problem instead.\end{remark}

\section{Experiment Details}
\label{sec:expdetails}
In this section we provide greater details on the experimental setup and implementations from Section~\ref{sec:experiments}. All experiments are run on Amazon EC2 clusters of m3.xlarge machines, with one core per machine. The code for each method is written in \textsf{\small Apache Spark}, v1.5.0. Our code is open-source and publicly available at: \href{http://github.com/gingsmith/proxcocoa}{\texttt{github.com/gingsmith/proxcocoa}}.

\paragraph{ADMM} Alternating Direction Method of Multipliers (ADMM) \cite{Boyd:2010bw} is a popular method that lends itself naturally to the distributed environment. Implementing ADMM for the problems of interest requires solving a large linear system $Cx=d$ on each machine, where $C \in \R^{n \times n}$ with $n$ scaling beyond $10^7$ for the datasets in Table~\ref{tab:datasets}, and with $C$ being possibly dense. It is prohibitively slow to solve this directly on each machine, and we therefore employ the iterative method of conjugate gradient with early stopping (see, e.g., \citep[Section 4.3]{Boyd:2010bw}). We further improve performance by using a varying rather than constant penalty parameter, as suggested in \citep[Section 3.4.1]{Boyd:2010bw}. \vspace{-2mm}

\paragraph{Mini-batch SGD and Proximal GD} Mini-batch SGD is a standard and widely used method for parallel and distributed optimization. We use the optimized code provided in Spark's machine learning library, MLlib, v1.5.0. We tune both the size of the mini-batch and the SGD step size using grid search. Proximal gradient descent can be seen as a specific setting of mini-batch SGD, where the mini-batch size is equal to the total number of datapoints. We thus also use the implementation in MLlib for prox-GD, and tune the step size parameter using grid search. \vspace{-2mm}

\paragraph{Mini-batch CD} Mini-batch CD aims to improve mini-batch SGD by employing coordinate descent, which has encouraging theoretical and practical backings \cite{ShalevShwartz:2011vo,Fercoq:2015kd,Tappenden:2015vha}. We implement mini-batch CD in Spark and scale the updates made at each round by $\frac{\beta}{b}$ for mini-batch size $b$ and $\beta \in [1,b]$, tuning both parameters $b$ and $\beta$ via grid search.  \vspace{-2mm}

\paragraph{Shotgun} As a special case of mini-batch CD, Shotgun \cite{Bradley:2011wq} is a popular method for parallel optimization. Shotgun can be seen an extreme case of mini-batch CD where the mini-batch is set to $1$ element per machine, 
i.e., there is a single update made by each machine per round. We see in the experiments that communicating this frequently becomes prohibitively slow in the distributed environment. \vspace{-2mm}

\paragraph{OWL-QN} OWN-QN \cite{Yu:2010vw} is a quasi-Newton method optimized in Spark's spark.ml package. Outer iterations of OWL-QN make significant progress towards convergence, but the iterations themselves can be slow because they require processing the entire dataset. \proxcocoa, the mini-batch methods, and ADMM with early stopping all improve on this by allowing the flexibility of only a subset of the dataset to be processed at each iteration. \proxcocoa and ADMM have even greater flexibility by allowing internal methods to process the dataset more than once. \proxcocoa makes this approximation quality specific, both in theoretical convergence rates and by providing general guidelines for setting the parameter. \vspace{-2mm}

\paragraph{\proxcocoa} We implement \proxcocoa with coordinate descent as a local solver. We note that since the framework and theory allow any internal solver to be used, \proxcocoa could benefit even beyond the results shown, by using existing fast $\Lone$-solvers for the single-machine case, such as \glmnet variants~\cite{Friedman:2010wm} or \textsc{blitz}~\cite{Johnson:2015tq}. The only parameter necessary to tune for \proxcocoa is the level of approximation quality, which we parameterize in the experiments using $H$, the number of local iterations of the iterative method run locally. Our theory relates local approximation quality to global convergence, and we provide a guideline for how to choose this value in practice that links the value to the systems environment at hand (Remark~\ref{rem:localtime}). We implement \cocoap as a special case of \proxcocoa for elastic net regularized objectives by mapping the main objective to~\eqref{eq:dualP} according to the steps described in Section~\ref{sec:oldcocoa}, and again use coordinate descent as a local solver. \vspace{-2mm}

\end{document}